\documentclass[conference]{IEEEtran}
\IEEEoverridecommandlockouts
\usepackage{cite}
\usepackage{amsmath,amssymb,amsfonts}
\usepackage{algorithmic}
\usepackage{graphicx}
\usepackage{textcomp}
\usepackage{xcolor}
\usepackage{bm}
\usepackage{mathtools, cuted}
\usepackage{url} % 堀之内追加
\usepackage{subcaption} % 中原追加
\def\BibTeX{{\rm B\kern-.05em{\sc i\kern-.025em b}\kern-.08em
    T\kern-.1667em\lower.7ex\hbox{E}\kern-.125emX}}
\begin{document}

\newtheorem{Theorem}{Theorem}
\newtheorem{Proposition}{Proposition}
\newtheorem{Definition}{Definition}
\newtheorem{Lemma}{Lemma}
\newtheorem{Corollary}{Corollary}
\newtheorem{Remark}{Remark}
\newtheorem{Example}{Example}
\newcommand{\argmax}{\mathop{\rm arg~max}\limits}
\newcommand{\argmin}{\mathop{\rm arg~min}\limits}

\title{A Mixture Autoregressive Image Generative Model on Quadtree Regions for Gaussian Noise Removal via Variational Bayes and Gradient Methods 
\thanks{This work was supported by JSPS KAKENHI Grant Numbers JP25K07732, JP26H02489, JP23K03863, JP23K04293, JP24H00370, JP26K17386, and JP26K06473. This paper is also a part of the outcome of research performed under a Waseda University Grant for Special Research Projects (Project number: 2026C-683).}
}

% 中央揃えでの著者情報の改行のため追記
\makeatletter
\newcommand{\linebreakand}{%
  \end{@IEEEauthorhalign}
  \hfill\mbox{}\par
  \mbox{}\hfill\begin{@IEEEauthorhalign}
}
\makeatother

\author{
\IEEEauthorblockN{Shota Saito}
\IEEEauthorblockA{%\textit{Faculty of Informatics} \\
\textit{Gunma University}\\
shota.s@gunma-u.ac.jp}
\and
\IEEEauthorblockN{Yuta Nakahara}
\IEEEauthorblockA{%\textit{Center for Data Science} \\
\textit{Waseda University}\\
y.nakahara@waseda.jp}
\and
\IEEEauthorblockN{Kohei Horinouchi}
\IEEEauthorblockA{%\textit{Center for Data Science} \\
\textit{Waseda University}\\
horinouchi@aoni.waseda.jp}
\linebreakand % 中央揃えでの著者情報の改行
\IEEEauthorblockN{Naoki Ichijo}
\IEEEauthorblockA{%\textit{Dept. of Applied Mathematics} \\
\textit{Energy Pool Japan}\\
naoki.ichijo@energy-pool.jp}
\and
\IEEEauthorblockN{Manabu Kobayashi}
\IEEEauthorblockA{%\textit{Center for Data Science} \\
\textit{Waseda University}\\
mkoba@waseda.jp}
\and
\IEEEauthorblockN{Toshiyasu Matsushima}
\IEEEauthorblockA{%\textit{Dept. of Applied Mathematics} \\
\textit{Waseda University}\\
toshimat@waseda.jp}
}

\maketitle

%The abstract should be limited to 300 words.
\begin{abstract}
This paper addresses the problem of image denoising for grayscale images. We propose a probabilistic image generative model that combines a quadtree region-partitioning model with a mixture autoregressive model, and propose a framework that reduces MAP (maximum a posteriori)-estimation-based denoising to the maximization of a variational lower bound. To maximize this lower bound, we develop an algorithm that alternately applies variational Bayes and gradient methods. We particularly demonstrate that the gradient-based update rule can be computed analytically without numerical computation or approximation. We carried out some experiments to verify that the proposed algorithm actually removes image noise and to identify directions for future improvement.
\end{abstract}

%%%%%
\section{Introduction}
%%%%%

In this study, we address the problem of image denoising.
Image denoising has been one of the central research topics in image processing from its early history to the present day.
The mathematical formulations developed in this field have served as a foundation for broad image processing tasks.
For example, in the field of image recovery, a framework known as Plug-and-Play \cite{venkatakrishnan2013plug} has been proposed, which enables us to utilize denoising techniques for various inverse problems such as inpainting and CT reconstruction.
In the field of representation learning, the denoising autoencoder \cite{vincent2008extracting} is regarded as an early example of masked self-supervised learning.
Furthermore, in the field of image generation, diffusion models \cite{ho2020denoising} achieve high-quality image generation by iteratively applying denoising procedure.

Conventional methods for image denoising can be organized along two axes: 1) the target range of probabilistic modeling within the observation and generation process of pixel values, and 2) the amount of data required for training.

Regarding the first axis, we have three categories. In category 1-a), no explicit probabilistic model is assumed. For example, Gaussian filters are included in this category. The methods in category 1-b) represent the image degradation process as a probabilistic model but introduce only a regularization term corresponding to a prior distribution over local pixel relationships for the image generation process. BM3D \cite{dabov2007image} and total variation regularization \cite{rudin1992nonlinear} are examples of this category.
Finally, in category 1-c), the entire image generation process, including both the degradation process and global structure, is represented as a probabilistic model such as variational auto-encoders \cite{kingma2014auto} and diffusion models \cite{ho2020denoising}.
The model proposed in this study also falls into category 1-c).

Organizing these methods further along the second axis, generative models based on deep neural networks\cite{kingma2014auto, ho2020denoising}, which are cited as examples of 1-c), require large-scale image datasets to learn the prior distribution. We call this category 2-a).
In contrast, methods cited as examples of 1-a) and 1-b), such as Gaussian filters, \cite{dabov2007image}, and \cite{rudin1992nonlinear}, require no prior training and work on a single noisy image alone. We call this category 2-b).
The method proposed in this study belongs to category 2-b).
That is, the proposed method is characterized by representing the entire image observation and generation process as a probabilistic model, as in 1-c), while working on a single image as in 2-b).

To achieve these characteristics, we propose a framework that reduces MAP-estimation-based denoising, similar to Plug-and-Play \cite{venkatakrishnan2013plug}, to the maximization of a variational lower bound.
Although this framework is applicable to a variety of degradation processes and probabilistic image generative models, we here adopt pixel-wise independent Gaussian noise as an example of the degradation process, and propose a novel probabilistic image generative model that combines a quadtree region-partitioning model with a mixture autoregressive model as an example of the probabilistic image generative model.
Since no efficient method for computing the variational lower bound for such a model has been proposed yet, we derive a computation method based on variational Bayes (VB).
Using this, we propose a denoising algorithm that alternately updates a provisional restored image via a gradient method and updates the variational lower bound under the current provisional restored image.
In particular, we show that the gradient-based update rule can be computed analytically, without numerical computation or approximation.

The primary contribution of this paper is to propose a new and flexible framework for image denoising, but we carried out some numerical experiments to check the behavior of the proposed method.

%%%%%
\section{Proposed Model} \label{section_proposed_model}
%%%%%

Figure \ref{fig:BasicModel} illustrates our proposed probabilistic image generative model. As shown in Fig.\ \ref{fig:BasicModel}, let the height and width of an image be $h \in \mathbb{N}$ and $w \in \mathbb{N}$, respectively. The set of possible pixel values is denoted by $\mathcal{V}$. In this study, we deal with a grayscale image, and we set $\mathcal{V} = \mathbb{R}$. In Appendix A, we show the schematic overview of the problem setup.
%In the following, we describe the prior distributions in Sections \ref{section_tree_prior}-–\ref{section_label_prior}, the image generation probabilistic model in Section \ref{section_generative_model}, the degradation process in Section \ref{section_degraded}, and the Bayes optimal decision in Section \ref{section_Bayes_decision}. 

\begin{figure}[t]
    \centering
    \includegraphics[width=\linewidth]{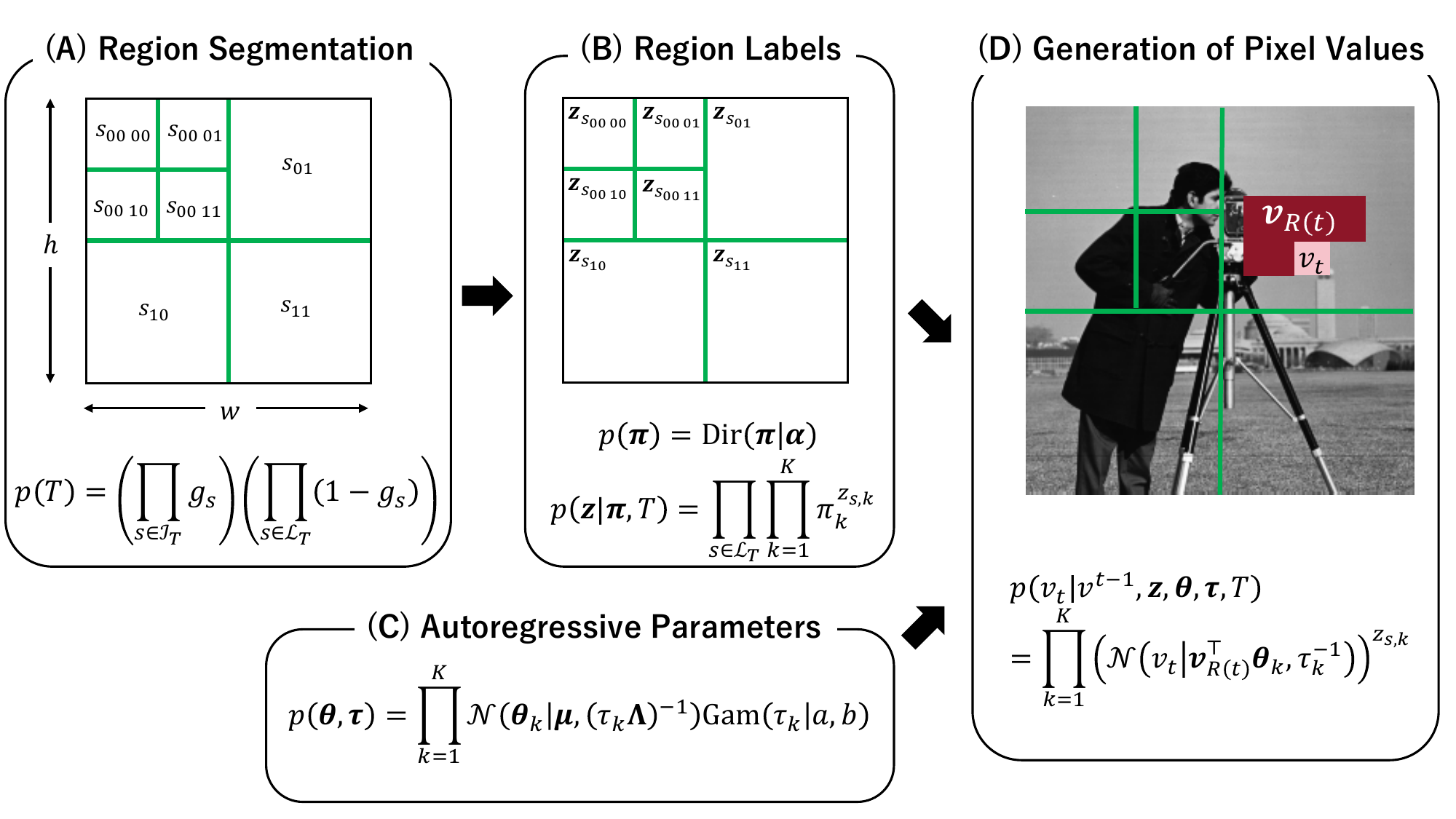}
    \caption{Proposed probabilistic image generative model}
    \label{fig:BasicModel}
\end{figure}

%%%%%
\subsection{Prior distribution of region segmentation} \label{section_tree_prior}
%%%%%

As shown in (A) in Fig.\ \ref{fig:BasicModel}, we assume that the region segmentation of an image is determined according to a prior distribution. Since a quadtree corresponds to the region segmentation, we first introduce notations for a quadtree.

Let $D_\mathrm{max} \in \mathbb{N}$ denote the maximum depth of the tree, and denote a rooted complete quadtree of depth $D_\mathrm{max}$ by $T_\mathrm{max}$. The tree $T_\mathrm{max}$ corresponds to the finest region segmentation pattern. Let the sets of all nodes, internal nodes, and leaf nodes of $T_\mathrm{max}$ be denoted by $\mathcal{S}_\mathrm{max}$, $\mathcal{I}_\mathrm{max}$, and $\mathcal{L}_\mathrm{max}$, respectively ($\mathcal{S}_\mathrm{max} = \mathcal{I}_\mathrm{max} \cup \mathcal{L}_\mathrm{max}$). Among the rooted subtrees of $T_\mathrm{max}$, the set of all subtrees that contain the root node of $T_\mathrm{max}$ is denoted by $\mathcal{T}$. 

For a quadtree $T \in \mathcal{T}$, its node set, internal node set, and leaf node set are denoted by $\mathcal{S}_T$, $\mathcal{I}_T$, and $\mathcal{L}_T$, respectively ($\mathcal{S}_T = \mathcal{I}_T \cup \mathcal{L}_T$). A segmented region of an image is represented by a leaf node $s \in \mathcal{L}_T$. Hence, when we say ``a node $s$,'' we mean that $s$ is a segmented region of an image.

We assume the following prior distribution for $T \in \mathcal{T}$:
\begin{align}
    p(T) = \left( \prod_{s \in \mathcal{I}_T} g_s \right) \left( \prod_{s \in \mathcal{L}_T} (1-g_s) \right), \label{eq_prior_T}
\end{align}
where $g_s \in [0,1]$ is a hyperparameter and we set $g_s = 0$ for $s \in \mathcal{L}_\mathrm{max}$.

%%%%%
\begin{Remark}
    The prior in \eqref{eq_prior_T} was introduced by \cite{CTW_Matsu_07}, \cite{CTW_Matsu_09}, and the properties of \eqref{eq_prior_T}, e.g., $\sum_{T \in \in \mathcal{T}} p(T) = 1$, were summarized in \cite{full_rooted_trees}. After these papers, the prior \eqref{eq_prior_T} was revisited by \cite{PK_22}, \cite{PK_24}. Regarding historical background, see \cite{SoftBCT}.
\end{Remark}
%%%%%

%%%%%
\subsection{Prior distribution of region labels} \label{section_label_prior}
%%%%%

As shown in (B) in Fig.\ \ref{fig:BasicModel}, each segmented region is assigned one of $K \in \mathbb{N}$ labels. Let $\bm z_s \coloneqq (z_{s,1}, z_{s,2}, \ldots, z_{s,K})^\top \in \{ 0, 1\} ^K$ is a one-hot vector, where $z_{s,k}$ is 1 if the label assigned to a node $s$ is $k$, and 0 otherwise. Let $\bm z \coloneqq \{ \bm z_s \}_{s \in \mathcal{L}_T}$. Moreover, let $\bm \pi \coloneqq (\pi_1, \pi_2, \ldots, \pi_K)^\top \in [0, 1]^K$ be the vector consisting of the probabilities $\pi_k$ that $z_{s,k} = 1$. That is, for each $k = 1, 2, \ldots, K$, $0 \leq \pi_k \leq 1$ and $\sum_{k=1}^K \pi_k = 1$. In this study, we assume $p(\bm \pi) = \mathrm{Dir} (\bm \pi | \bm \alpha)$ and $p(\bm z | \bm \pi, T) = \prod_{s \in \mathcal{L}_T} \prod_{k=1}^K \pi_k^{z_{s,k}}$,
where $\bm \alpha \coloneqq (\alpha_1, \alpha_2, \ldots, \alpha_K)^\top\in \mathbb{R}_{>0}^K$ is the parameter of the Dirichlet prior distribution.

%%%%%
\subsection{Prior distribution of autoregressive parameters} \label{section_parameter_prior}
%%%%%

As shown in (C) in Fig.\ \ref{fig:BasicModel}, the parameters of the autoregressive model follow a Gaussian-Gamma prior distribution: $    p(\bm \theta, \bm \tau) = \prod_{k=1}^K \mathcal{N}(\bm \theta_k | \bm \mu, (\tau_k \bm \Lambda)^{-1}) \mathrm{Gam}(\tau_k | a, b),$
where there are $K$ candidate pairs of coefficients and precision parameters for the autoregressive model, $(\bm \theta_1, \tau_1), \dots , (\bm \theta_K, \tau_K) \in \mathbb{R}^D \times \mathbb{R}_{>0}$, and let $\bm \theta \coloneqq \{ \bm \theta_k \}_{k=1}^K$ and $\bm \tau \coloneqq \{ \tau_k \}_{k=1}^K$. Furthermore, $D \in \mathbb{N}$ is the dimension of the coefficients of the autoregressive model, $\bm \mu \in \mathbb{R}^D$ is the mean vector of the Gaussian prior distribution, $\bm \Lambda \in \mathbb{R}^{D \times D}$ is the precision matrix (positive definite symmetric matrix) of the Gaussian distribution, and $a, b \in \mathbb{R}_{>0}$ are the parameters of the Gamma prior distribution.

%%%%%
\subsection{Probabilistic image generative model} \label{section_generative_model}
%%%%%

As shown in (D) in Fig.\ \ref{fig:BasicModel}, we assume that pixel values are generated by an autoregressive model. Let $v_t \in \mathcal{V}$ denote the $t$-th pixel value in raster scan order, and let $v^{t-1} \in \mathcal{V}^{t-1}$ denote the sequence of pixel values before $t$ in raster scan order. Furthermore, let $s$ denote the node that contains $v_t$. Then, the pixel value $v_t$ follows the probability distribution
\begin{align*}
    p(v_t | v^{t-1}, \bm z, T, \bm \theta, \bm \tau) = \prod_{k=1}^K \mathcal{N}(v_t | \bm v_{R(t)}^\top \bm \theta_k, \tau_k^{-1})^{z_{s,k}},
\end{align*}
where ${R(t)}$ is the set of position indices of the $(D-1)$ neighboring pixels surrounding $v_t$, and $\bm v_{R(t)}$ is a vector consisting of the pixel values within $R(t)$, with the constant term 1 appended to the end. The position indices of the neighboring pixels are determined in advance, and if $R(t)$ refers to a location outside the image, the corresponding component of $\bm v_{R(t)}$ at that position is set to a constant (for example, see Appendix B). We refer to $\bm v_{R(t)}$ as the \emph{reference pixel vector} of the pixel value $v_t$.

Also, let $h_s$ and $w_s$ be the height and width of the node $s$, respectively, and let $\bm v_s \in \mathcal{V}^{h_s w_s}$ be the vector whose components are pixel values of the node $s$. Let $\bm V_s \in \mathcal{V}^{h_s w_s \times D}$ be the matrix obtained by transposing the reference pixel vectors of all pixel values contained in $s$ and arranging them in rows (see (34) in Appendix H). Then, the likelihood for the entire image $\bm v$, where $\bm v \in \mathcal{V}^{h w}$ is the vector containing all pixel values, is given by
\begin{align*}
        p(\bm v | \bm z, T, \bm \theta, \bm \tau) = \prod_{s \in \mathcal{L}_T} \prod_{k=1}^K \mathcal{N}(\bm v_s | \bm V_s \bm \theta_k, (\tau_k \bm I_s)^{-1} )^{z_{s, k}}.
\end{align*}
Here, $\bm I_s$ is the identity matrix of size $h_s w_s$.

%%%%%
\subsection{Degradation process} \label{section_degraded}
%%%%%

Hereinafter, instead of using the notation $v_t$, we denote the pixel value of the $i$-th row and $j$-th column by $v_{i,j}$. We use the same notation for $R(t)$. The original image is observed as the observed image $\bm v'$ through the degradation process $p(\bm v' | \bm v)$. We assume pixel-wise independent Gaussian noise:
\begin{align}
    p(\bm v' | \bm v) = \prod_{i=1}^h \prod_{j=1}^w \mathcal{N}(v'_{i,j}|v_{i,j},\sigma^2), \label{eq_gaussian_noise}
\end{align}
where the variance $\sigma^2 \in \mathbb{R}_{>0}$ is assumed to be known.

%%%%%
\subsection{Bayes optimal decision} \label{section_Bayes_decision}
%%%%%

Under the above setup, the goal is to restore the original image $\bm v$ when the observed image $\bm v'$ is obtained. Let $\hat{\bm v}$ denote the estimate of $\bm v$. When we assume the 0-1 loss $I \{ \hat{\bm v} \neq \bm v \}$, where $I\{\cdot \}$ denotes the indicator function, the optimal decision $\hat{\bm v}^*$ under the Bayes criterion is given by the MAP (maximum a posteriori) estimate (see, e.g., \cite{berger1985statistical}):
\begin{align}
    \hat{\bm v}^* = \mathrm{arg} \max_{\bm v} p(\bm v | \bm v')= \mathrm{arg} \max_{\bm v} \left \{ \ln p(\bm v' | \bm v) + \ln p(\bm v) \right \}. \label{eq_Bayes_optimal}
\end{align}

%%%%%
\section{Proposed Algorithm} \label{section_proposed_algorithm}
%%%%%

Since it is difficult to compute \eqref{eq_Bayes_optimal} directly, we consider the following lower bound for $\ln p(\bm v)$:
\begin{align}
    \ln p(\bm v) \geq \text{VL}(p(\bm v, \bm z, T, \bm \theta, \bm \tau, \bm \pi), q_{\bm v} (\bm z, T, \bm \theta, \bm \tau, \bm \pi)),  \label{eq_lower_bound}
\end{align}
where $q_{\bm v} (\bm z, T, \bm \theta, \bm \tau, \bm \pi)$ is an approximate posterior distribution of $\bm z, T, \bm \theta, \bm \tau, \bm \pi$ for a given $\bm v \in \mathcal{V}^{h w}$, and the variational lower bound $ \mathrm{VL}(p(\bm v, \bm z, T, \bm \theta, \bm \tau, \bm \pi), q_{\bm v} ( \bm z, T, \bm \theta, \bm \tau, \bm \pi))$ is 
\begin{align*}
     & \mathrm{VL}(p(\bm v, \bm z, T, \bm \theta, \bm \tau, \bm \pi), q_{\bm v} ( \bm z, T, \bm \theta, \bm \tau, \bm \pi))  \nonumber \\
     & \quad \coloneqq \mathbb{E}_{q_{\bm v} ( \bm z, T, \bm \theta, \bm \tau, \bm \pi))} \left[ \ln \frac{p(\bm v, \bm z, T, \bm \theta, \bm \tau, \bm \pi)}{q_{\bm v} ( \bm z, T, \bm \theta, \bm \tau, \bm \pi))} \right].
\end{align*}

Based on the lower bound \eqref{eq_lower_bound}, we compute an approximate value for $\hat{\bm v}^*$ using the following iterative optimization algorithm.
\begin{enumerate}
    \item[(i)] Using the gradient method, we update $\bm v$ by maximizing
    \begin{align}
        \ln p(\bm v' | \bm v) + \mathrm{VL}(p(\bm v, \bm z, T, \bm \theta, \bm \tau, \bm \pi), q_{\bm v} ( \bm z, T, \bm \theta, \bm \tau, \bm \pi)). \label{objective_function}
    \end{align}
    \item[(ii)] Given $\bm v$, we update the approximate posterior distribution $q_{\bm v} (\bm z, T, \bm \theta, \bm \tau, \bm \pi)$ using VB.
    \item[(iii)] We repeat steps (i) and (ii) until an appropriate convergence criterion is satisfied.
\end{enumerate}

%%%%%
\subsection{Update formula for $q_{\bm v} (\bm z, T, \bm \theta, \bm \tau, \bm \pi)$} \label{section_variational_Bayes}
%%%%%

First, we explain the update formula for step (ii). Hereinafter, we denote $q_{\bm v}(\cdot)$ as $q(\cdot)$ and assume the following factorization: $q(\bm z, T, \bm \theta, \bm \tau, \bm \pi) = q(\bm z, T) q(\bm \theta, \bm \tau, \bm \pi)$.

The VB minimizes the KL divergence between the approximate posterior and the true posterior. The KL divergence $\mathrm{KL}(q(\bm z, T, \bm \theta, \bm \tau, \bm \pi) \| p(\bm z, T, \bm \theta, \bm \tau, \bm \pi | \bm v))$ is minimized by the approximate posterior distribution $q^*$ that satisfies the following (see, e.g., \cite{bishop}):
\begin{align}
& \hspace{-2mm} \ln q^* (\bm z, T) = \mathbb{E}_{q^*(\bm \theta, \bm \tau, \bm \pi)} \left[ \ln p(\bm v, \bm z, T, \bm \theta, \bm \tau, \bm \pi) \right] + \mathrm{const.}, \label{q_star_z_T}\\
& \hspace{-2mm} \ln q^* (\bm \theta, \bm \tau, \bm \pi) = \mathbb{E}_{q^*(\bm z, T)} \left[ \ln p(\bm v, \bm z, T, \bm \theta, \bm \tau, \bm \pi) \right] + \mathrm{const.} \label{q_star_theta_tau_pi}
\end{align}

However, $q^* (\bm z, T)$ and $q^* (\bm \theta, \bm \tau, \bm \pi)$ depend on each other. Therefore, we update $q(\bm z, T)$ and $q(\bm \theta, \bm \tau, \bm \pi)$ in turn from an initial value until convergence. The specific update formulas are given as in the following propositions.

\begin{Proposition} \label{Prop_update_z_T}
The posterior distribution $q(\bm z, T)$ can be decomposed as $q(T) \prod_{s \in \mathcal{L}_T} q(\bm z_s)$, and 
\begin{align*}
   q(\bm z_s) = \prod_{k=1}^K (\pi'_{s,k})^{z_{s,k}}, q(T) = \left( \prod_{s \in \mathcal{I}_T} g'_s \right) \left( \prod_{s \in \mathcal{L}_T} (1-g'_s) \right).
\end{align*}
Here, $\pi'_{s,k} \coloneqq \rho_{s,k}/\sum_{k=1}^K \rho_{s,k}$ and
\begin{align}
    &\ln \rho_{s,k} \coloneqq \psi(\alpha'_k) - \psi \left( \textstyle \sum_{k=1}^K \alpha'_k \right) \nonumber \\
    & + \frac{h_s w_s}{2} \left( -\ln 2\pi + \psi (a'_k) - \ln b'_k \right) \nonumber \\
    &- \frac{a'_k}{2 b'_k} (\bm v_s - \bm V_s \bm \mu'_k)^\top (\bm v_s - \bm V_s \bm \mu'_k) -\frac{1}{2} \mathrm{Tr} \{ \bm V_s^\top \bm V_s (\bm \Lambda'_k)^{-1} \}, \label{eq:rho_closed_form}
\end{align}
where $\psi (\cdot)$ denotes the digamma function, $\mathrm{Tr} \{\cdot\}$ denotes the trace of a matrix, $\alpha'_k$, $\bm \Lambda'_k$, $\bm \mu'_k$, $a'_k$, and $b'_k$ are given by \eqref{def_alpha_prime}, \eqref{def_lambda_prime}, \eqref{def_mu_prime}, \eqref{def_a_prime}, and \eqref{def_b_prime}, respectively, and
\begin{align}
    g'_s &\coloneqq 
    \begin{cases}
        \frac{g_s \prod_{s_\mathrm{ch} \in \mathrm{Ch}(s)} \phi_{s_\mathrm{ch}}}{\phi_s}, & s \in \mathcal{I}_\mathrm{max}, \\
        0, & s \in \mathcal{L}_\mathrm{max},
    \end{cases} \label{def_g_prime} \\
    \phi_s &\coloneqq \nonumber \\
    & \begin{cases}
        (1-g_s) \sum_{k=1}^K \rho_{s,k} + g_s \prod_{s_\mathrm{ch} \in \mathrm{Ch}(s)} \phi_{s_\mathrm{ch}}, & s \in \mathcal{I}_\mathrm{max}, \\
        \sum_{k=1}^K \rho_{s,k}, & s \in \mathcal{L}_\mathrm{max},
    \end{cases} \nonumber
\end{align}
where $\mathrm{Ch}(s)$ denotes the set of child nodes of a node $s$.
\end{Proposition}

\begin{IEEEproof}
    See Appendix C.
\end{IEEEproof}

Let $q(s \in \mathcal{L}_T) \coloneqq \mathbb{E}_{q(T)}[I \{ s \in \mathcal{L}_T \}]$. That is, $q(s \in \mathcal{L}_T)$ denotes the probability that a node $s$ is a leaf node under $q(T)$. Then, from \cite[Theorem 2]{full_rooted_trees}, it holds that
\begin{align}
q(s \in \mathcal{L}_T) = (1-g'_s) \left( \prod_{s' \prec s} g'_{s'} \right), \label{eq_expression_q}
\end{align}
where $s' \prec s$ denotes that $s'$ is an ancestor of $s$. 

\begin{Proposition} \label{Prop_update_pi_theta_tau}
The posterior distribution $q(\bm \theta, \bm \tau, \bm \pi)$ can be decomposed as $q(\bm \pi) \prod_{k=1}^K q(\bm \theta_k, \tau_k)$, and
    \begin{align*}
        q(\bm \pi) &= \mathrm{Dir}(\bm \pi | \bm \alpha'), \\
        q(\bm \theta_k, \tau_k) &= \mathcal{N}(\bm \theta_k | \bm \mu'_k, (\tau_k \bm \Lambda'_k)^{-1}) \mathrm{Gam}(\tau_k | a'_k, b'_k),
    \end{align*}
    where $\bm \alpha' \coloneqq (\alpha'_1, \ldots, \alpha'_K)^\top \in \mathbb{R}_{>0}^K$, and for $k=1,\ldots,K$,
    \begin{align}
        \alpha'_k &\coloneqq \alpha_k + \sum_{s \in \mathcal{S}_\mathrm{max}} q(s \in \mathcal{L}_T) \pi'_{s,k}, \label{def_alpha_prime} \\
        \bm \Lambda'_k &\coloneqq \bm \Lambda + \sum_{s \in \mathcal{S}_\mathrm{max}} q(s \in \mathcal{L}_T) \pi'_{s,k} \bm V_s^\top \bm V_s, \label{def_lambda_prime} \\
        \bm \mu'_k &\coloneqq \left( \bm \Lambda'_k \right)^{-1} \left( \bm \Lambda \bm \mu + \sum_{s \in \mathcal{S}_\mathrm{max}} q(s \in \mathcal{L}_T) \pi'_{s,k} \bm V_s^\top \bm v_s \right), \label{def_mu_prime} \\
        a'_k &\coloneqq a + \frac{1}{2}\sum_{s \in \mathcal{S}_\mathrm{max}} q(s \in \mathcal{L}_T) \pi'_{s,k} h_s w_s, \label{def_a_prime} \\
        b'_k &\coloneqq b + \frac{1}{2} \Bigl( \bm \mu^\top \bm \Lambda \bm \mu  \nonumber \\
        & \left. \quad + \sum_{s \in \mathcal{S}_\mathrm{max}} q(s \in \mathcal{L}_T) \pi'_{s,k} \bm v_s^\top \bm v_s - (\bm \mu'_k)^\top \bm \Lambda'_k \bm \mu'_k \right). \label{def_b_prime}
    \end{align}
\end{Proposition}

\begin{IEEEproof}
    See Appendix D.
\end{IEEEproof}

%%%%%
\subsection{Update formula for $\bm v$} \label{section_derivative}
%%%%%

We explain the update formula for step (i), i.e., the calculation for updating $\bm v$ using the gradient method to maximize \eqref{objective_function}. Specifically, let \eqref{objective_function} (the objective function) be denoted by $f(\bm v)$, and we explain the calculation of $\frac{\partial f(\bm v)}{\partial v_{i,j}}$. 

First, we introduce some notations.
\begin{enumerate}
\item As described in Section \ref{section_proposed_model}, $\bm v_{R(i,j)}$ denotes the \emph{reference pixel vector} of $v_{i,j}$. For example, $\bm v_{R(i,j)}$ for $D=4$ is 
$
    \bm v_{R(i,j)} \coloneqq  (v_{i, j-1}, v_{i-1, j+1}, v_{i-1, j}, 1)^\top
$
(see Fig.\ \ref{fig:region}).
\item For each position index in $R(i,j)$, let $\tilde{R}(i,j)$ denote the set of position indices that are point-symmetric with respect to $(i,j)$, i.e., the set of $(D-1)$ position indices of pixels that refer to $v_{i,j}$ when the pixels are examined in raster scan order. 
\item Let $\bm v_{s, \tilde{R}(i,j)}$ denote the $D$-dimensional vector obtained by concatenating the pixel values at the positions in $\tilde{R}(i,j)$ among the $\bm v_s$ and appending a zero at the end; this is called the \emph{reverse reference pixel vector} for the pixel value $v_{i,j}$ at a node $s$. For example, if $v_{i,j}$ is located sufficiently far from the edge of a node $s$ and $D=4$, 
\begin{align}
    \bm v_{s, \tilde{R}(i,j)} \coloneqq ( v_{i, j+1}, v_{i+1, j-1}, v_{i+1, j}, 0)^\top \label{eq_referenced_vector}
\end{align}
(see Fig.\ \ref{fig:region}), where the corresponding component of $\bm v_{s, \tilde{R}(i,j)}$ is set to 0 if $\tilde{R}(i,j)$ refers to a region outside a node $s$. 
\item The $D \times D$ matrix $\bm V_{s,\tilde{R}(i,j)}$ is defined by extracting the rows from $\bm V_s$ that correspond to the positions in $\tilde{R}(i,j)$, and appending the $D$-dimensional zero vector $\bm 0^\top$ to the last row. When $D=4$, 
\begin{align}
    \bm V_{s,\tilde{R}(i,j)} & \coloneqq 
    \begin{pmatrix}
        \bm v_{R(i, j+1)}^\top \\ \bm v_{R(i+1,j-1)}^\top \\ \bm v_{R(i+1,j)}^\top \\ \bm 0^\top
    \end{pmatrix},
\end{align}
where, if $\tilde{R}(i,j)$ refers to the exterior of a node $s$, the row vector at the corresponding position of $\bm V_{s, \tilde{R}(i,j)}$ is set to $\bm 0^\top$.
\item Let $\mathrm{path}(v_{i,j})$ denote the set of nodes $s$ on the path of $T_{\mathrm{max}}$ that contain $v_{i,j}$ (an example is in Fig.\ \ref{fig:example_path}).
\end{enumerate}

\begin{figure}[t]
    \centering
    \includegraphics[width=0.6\linewidth]{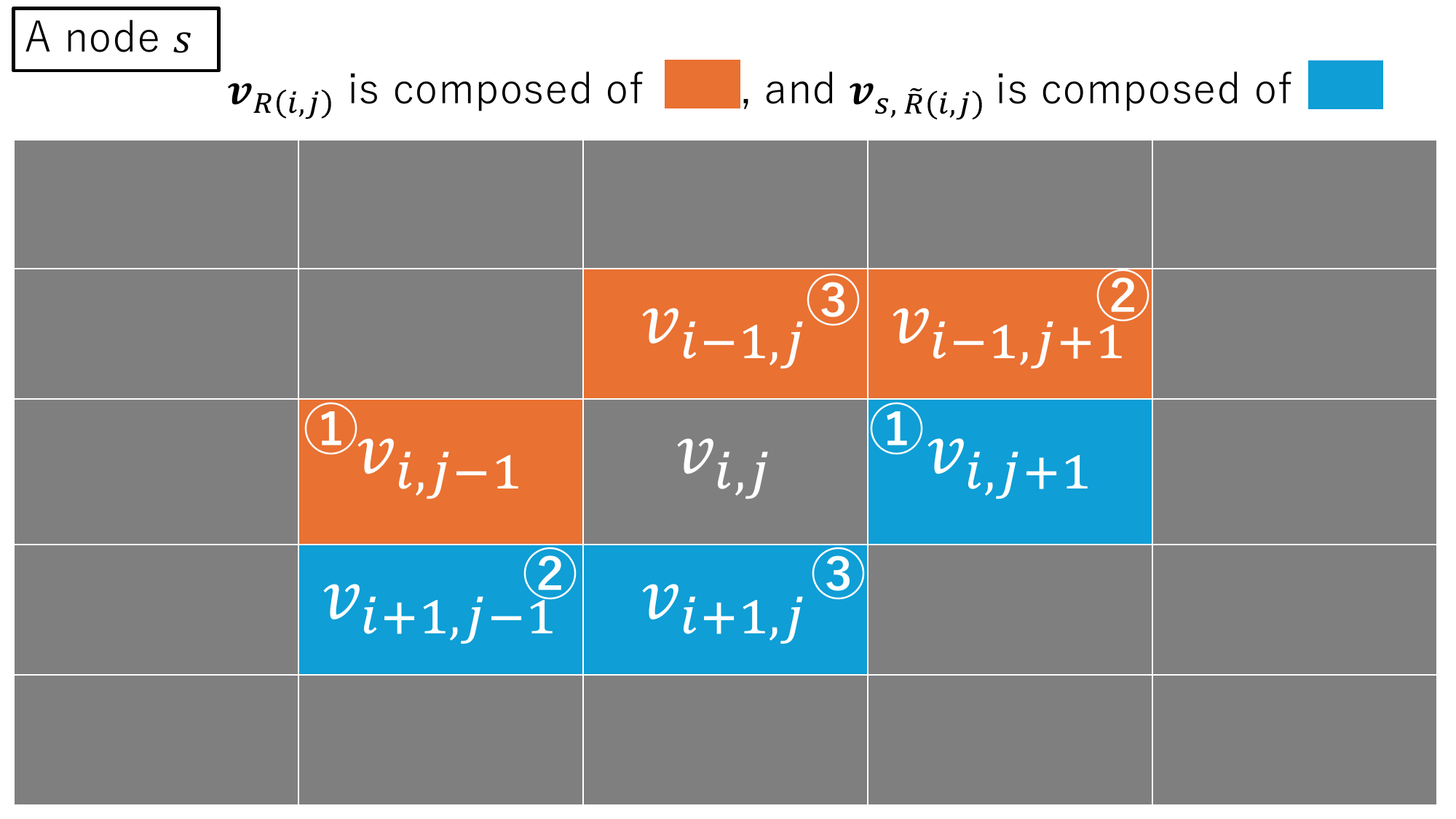}
    \caption{$\bm v_{R(i,j)}$ and $\bm v_{s,\tilde{R}(i,j)}$ when $D=4$}
    \label{fig:region}
\end{figure}

\begin{figure}[t]
    \centering
    \includegraphics[width=0.6\linewidth]{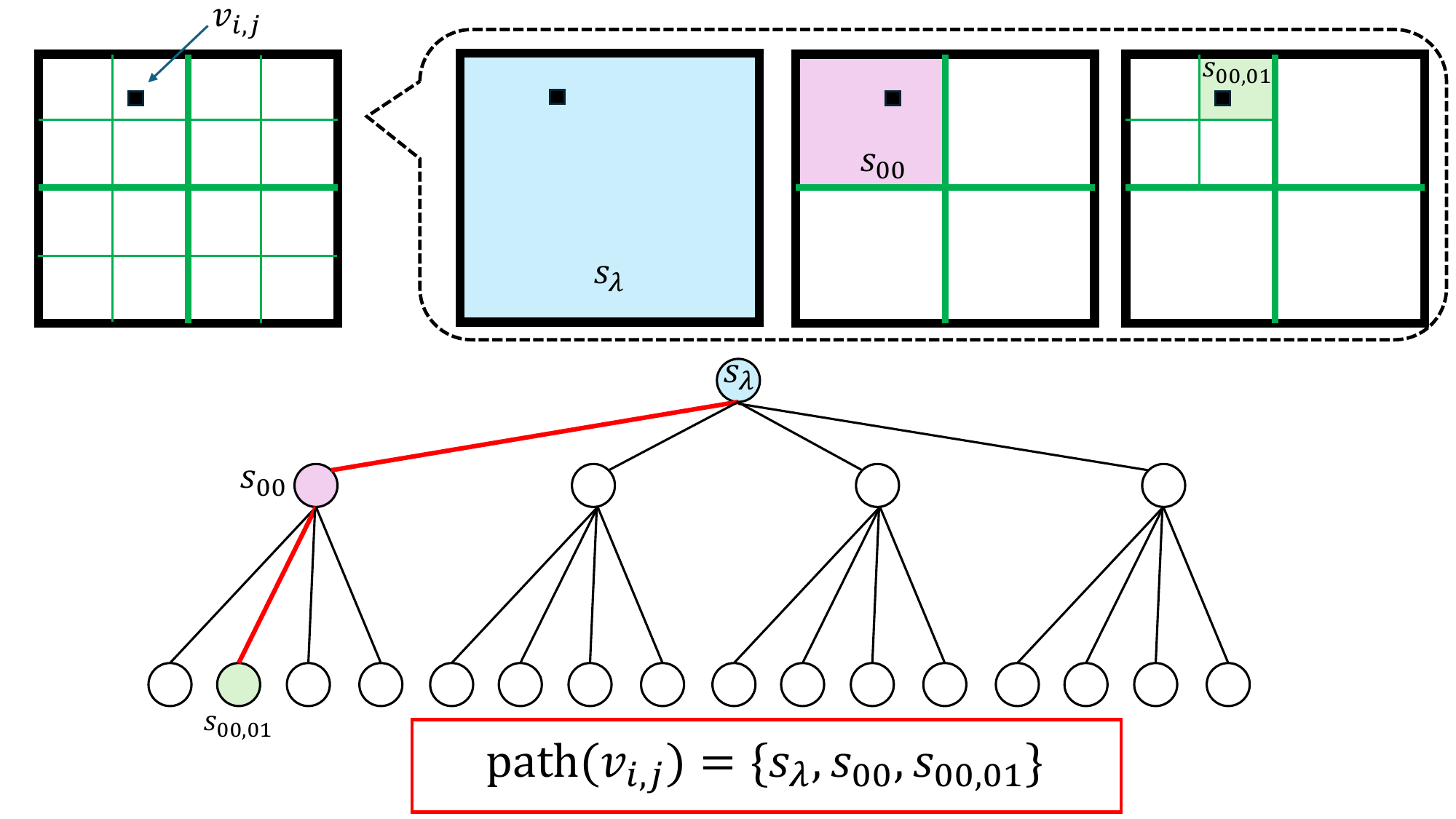}
    \caption{An example of $\mathrm{path}(v_{i,j})$ for $D_{\mathrm{max}}=2$}
    \label{fig:example_path}
\end{figure}

%The following proposition shows that $\frac{\partial f(\bm v)}{\partial v_{i,j}}$ is expressed by using the pixel value $v'_{i,j}$ of the observed image,  variance $\sigma^2$ of the Gaussian noise in the degradation process, parameters of the approximate posterior distribution $g'_s, \pi'_{s,k}, a'_k, b'_k, \bm \mu'_k, \bm \Lambda'_k$, reference pixel vector $\bm v_{R(i,j)}$, reverse reference pixel vector $\bm v_{s,\tilde{R}(i,j)}$, and matrix $\bm V_{s,\tilde{R}(i,j)}$. 

The derivative $\frac{\partial f(\bm v)}{\partial v_{i,j}}$ is expressed as follows.

\begin{Proposition} \label{Prop_update_v}
It holds that
\begin{align*}
\frac{\partial f(\bm v)}{\partial v_{i,j}} 
&= \frac{v'_{i,j} - v_{i,j}}{\sigma^2} + \sum_{s \in \mathrm{path}(v_{i,j})} (1-g'_s) \left(\prod_{s' \prec s} g'_{s'} \right) \nonumber \\
& \times \sum_{k=1}^K \pi'_{s,k} \biggl(- \frac{a'_k}{b'_k} \left \{ v_{i,j} - \left(\bm v_{R(i,j)} + \bm v_{s,\tilde{R}(i,j)} \right)^\top \bm \mu'_k \right. \nonumber \\
& \left. ~ + (\bm \mu'_k)^\top \bm V_{s,\tilde{R}(i,j)} \bm \mu'_k \right \} - \mathrm{Tr} \left \{ \bm V_{s,\tilde{R}(i,j)} (\bm \Lambda'_k)^{-1} \right \} \biggr).
\end{align*}
\end{Proposition}

\begin{IEEEproof}
    See Appendix H.
\end{IEEEproof}

%%%%%
\section{Numerical Experiments} \label{section_experiments}
%%%%%

% \subsection{Objectives}

The objectives of these experiments are twofold: to verify that the proposed framework actually removes image noise, and to identify directions for future improvement.

% \subsection{Experimental Setup}

\begin{figure*}
    \centering
    \begin{subfigure}[b]{0.11\linewidth}
        \includegraphics[clip, viewport=0 16 128 144, width=\linewidth]{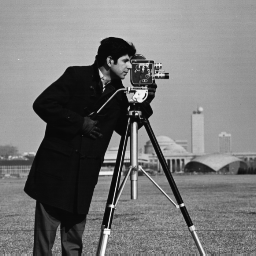}
        \caption{Clean}
    \end{subfigure}
    \hfill
    \begin{subfigure}[b]{0.11\linewidth}
        \includegraphics[clip, viewport=0 16 128 144, width=\linewidth]{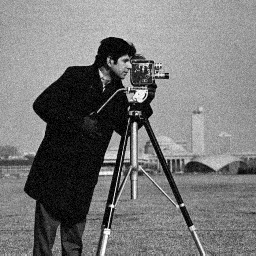}
        \caption{Noised}
    \end{subfigure}
    \hfill
    \begin{subfigure}[b]{0.11\linewidth}
        \includegraphics[clip, viewport=0 16 128 144, width=\linewidth]{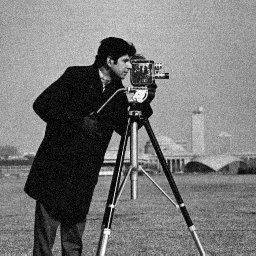}
        \caption{GF}
    \end{subfigure}
    \hfill
    \begin{subfigure}[b]{0.11\linewidth}
        \includegraphics[clip, viewport=0 16 128 144, width=\linewidth]{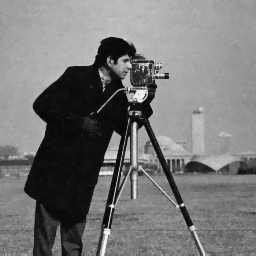}
        \caption{TV}
    \end{subfigure}
    \hfill
    \begin{subfigure}[b]{0.11\linewidth}
        \includegraphics[clip, viewport=0 16 128 144, width=\linewidth]{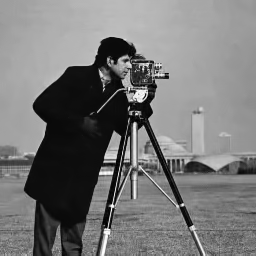}
        \caption{BM3D}
    \end{subfigure}
    \hfill
    \begin{subfigure}[b]{0.11\linewidth}
        \includegraphics[clip, viewport=0 16 128 144, width=\linewidth]{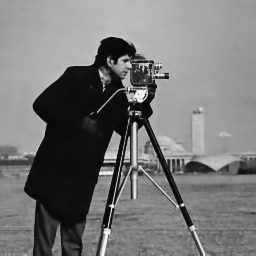}
        \caption{DIP}
    \end{subfigure}
    \hfill
    \begin{subfigure}[b]{0.11\linewidth}
        \includegraphics[clip, viewport=0 16 128 144, width=\linewidth]{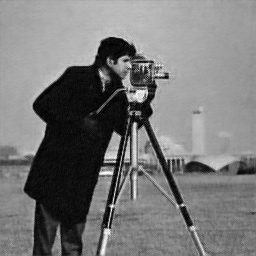}
        \caption{N2V}
    \end{subfigure}
    \hfill
    \begin{subfigure}[b]{0.11\linewidth}
        \includegraphics[clip, viewport=0 16 128 144, width=\linewidth]{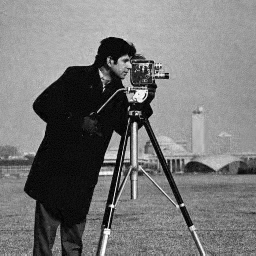}
        \caption{Proposed}
    \end{subfigure}
    \caption{A part of an image used in the experiment with $\sigma = 10$. From left to right, they show the clean image, the noisy image, the denoised images by GF, TV denoising, BM3D, DIP, N2V, and the proposed method.}
    \label{fig:visual_examples}
\end{figure*}

\noindent\textbf{Experimental setup.}
% \noindent\textbf{Dataset.}
Experiments were conducted on the Set12 benchmark (a standard grayscale image denoising benchmark used in \cite{ZhangTIP2017}\footnote{In our experiments, the clean Set12 images were taken from \url{https://github.com/cubeyoung/Noise2Score/tree/master/testdata/Set12/clean}.}). The noise levels were set to $\sigma \in \{5, 10, 15\}$. The value of $\sigma$ was treated as a known constant.

% \noindent\textbf{Baselines.}
The proposed method was compared against the following methods: Gaussian filtering (GF) implemented in scipy\cite{scipy}, total variation (TV) denoising \cite{rudin1992nonlinear} implemented in scikit-image\cite{scikit-image} as \texttt{denoise\_tv\_chambolle}, BM3D \cite{dabov2007image}, DIP \cite{Ulyanov2018DeepImagePrior}, and N2V \cite{Krull2019Noise2Void}. Note that DIP and N2V do not require $\sigma$. For GF, we used a Gaussian kernel whose standard deviations (the \texttt{sigma} option in \cite{scipy}) were $0.17$, $0.33$, and $0.50$ for the noise levels $\sigma = 5$, $10$, and $15$, respectively. The Gaussian kernel was truncated at three standard deviations. For TV denoising, the \texttt{weight} option in \cite{scikit-image} was set to $4.0$, $8.0$, and $9.0$ for $\sigma = 5$, $10$, and $15$, respectively.

% \noindent\textbf{Hyperparameters.}
The tree $T_\mathrm{max}$ grows until its depth reaches $D_\mathrm{max} = 30$ or the width or height of each leaf node region falls to 2 or below. The number of labels was set to $K = 100$, and the order of the autoregressive model to $D = 11$. For all noise levels, the hyperparameters of prior distributions were set to as follows: $g = 0.75$, $\bm{\alpha} = [0.01, \dots, 0.01]^\top$, $\bm{\mu} = \bm{0}$, $\bm{\Lambda} = \bm{I}$, $a = 1.0$, and $b = 100$.

% \noindent\textbf{Optimization method.}
Simple gradient ascent was used as the optimization method. 
% Preliminary small-scale experiments revealed the following: the objective function begins to decrease after a certain number of iterations; the learning rate has a substantial influence on this behavior; a larger noise level in the degradation process causes the objective function to begin decreasing earlier and to decrease more sharply; and continuing iterations for several consecutive steps after the objective function first begins to decrease, rather than stopping immediately, yields better values of performance metrics such as RMSE. Based on these observations,
Based on preliminary small-scale experiments, the step size at iteration $t$ was set to decay as $\eta_t = \frac{0.1\sigma}{1 + 0.05t}$, and early stopping was applied if the objective function decreased for 10 consecutive steps. The maximum number of iterations was set to 150.

% \noindent\textbf{Initialization and update frequency of VB.}
The initialization procedure was carried out once at the beginning of the algorithm as follows. The initial value of $\bm{v}$ was set to the observed image $\bm{v}'$. The hyperparameters of $q(\bm{\theta}, \bm{\tau}, \bm{\pi})$ were initialized as follows. We initialize $\bm{\alpha}'$ as $\bm{\alpha}$. Since $K = 100$, the observed image of height $h$ and width $w$ was divided into 100 rectangular regions by partitioning each axis into 10 equal parts. Denoting by $\bm{v}'_k$ the vector of pixel values belonging to the $k$-th region and by $\bm{V}'_k$ the matrix of corresponding reference pixel vectors, the hyperparameters $\bm{\Lambda}'_k$, $\bm{\mu}'_k$, $a'_k$, and $b'_k$ were initialized as $\bm{\Lambda}'_k = \bm{\Lambda} + (\bm{V}'_k)^\top \bm{V}'_k$, $\bm{\mu}'_k = \left(\bm{\Lambda}'_k\right)^{-1} \left(\bm{\Lambda}\bm{\mu} + (\bm{V}'_k)^\top \bm{v}'_k\right)$, $a'_k = a + \frac{1}{2} \cdot \frac{hw}{100}$, and $b'_k = b + \frac{1}{2}\left(\bm{\mu}^\top \bm{\Lambda} \bm{\mu}+ (\bm{v}'_k)^\top \bm{v}'_k - (\bm{\mu}'_k)^\top \bm{\Lambda}'_k \bm{\mu}'_k \right).$

%\begin{align*}
%    \bm{\Lambda}'_k &= \bm{\Lambda} + (\bm{V}'_k)^\top \bm{V}'_k, \\
%    \bm{\mu}'_k &= \left(\bm{\Lambda}'_k\right)^{-1}
%        \left(\bm{\Lambda}\bm{\mu} + (\bm{V}'_k)^\top \bm{v}'_k\right), \\
%    a'_k &= a + \frac{1}{2} \cdot \frac{hw}{100}, \\
%    b'_k &= b + \frac{1}{2}\left(
%        \bm{\mu}^\top \bm{\Lambda} \bm{\mu}
%        + (\bm{v}'_k)^\top \bm{v}'_k
%        - (\bm{\mu}'_k)^\top \bm{\Lambda}'_k \bm{\mu}'_k
%    \right).
%\end{align*}

Using these hyperparameters, the iteration starts by updating $q(\bm{z}, T)$. Once the gradient ascent for $\bm{v}$ began, $q(\bm{z}, T, \bm{\theta}, \bm{\tau}, \bm{\pi})$ at the previous iteration was used as the warm start for subsequent VB updates. The number of VB iterations performed per gradient ascent step was set to one.

% \subsection{Results}

\noindent\textbf{Results.}
Table \ref{tab:conventional_comparison_set12} shows the mean RMSE (Root Mean Square Error, smaller is better), PSNR (Peak Signal-to-Noise Ratio, larger is better), and SSIM\cite{SSIM} (Structural Similarity Index Measure, larger is better) over the images in Set12. Visual examples of denoised images are shown in Fig.\ \ref{fig:visual_examples}. When $\sigma$ is small, the proposed method is comparable to TV denoising. However, when $\sigma$ is large, the performance of the proposed method degrades.

% \subsection{Discussion}

\begin{figure}
    \centering
    \begin{subfigure}[b]{0.19\linewidth}
        \includegraphics[width=\linewidth]{clean.png}
        \caption{Clean}
    \end{subfigure}
    \hfill
    \begin{subfigure}[b]{0.19\linewidth}
        \includegraphics[width=\linewidth]{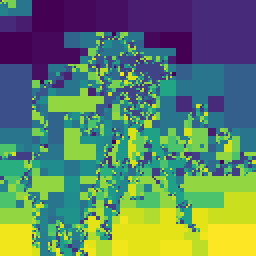}
        \caption{$\sigma = 5$}
    \end{subfigure}
    \hfill
    \begin{subfigure}[b]{0.19\linewidth}
        \includegraphics[width=\linewidth]{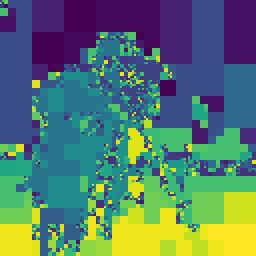}
        \caption{$\sigma = 10$}
    \end{subfigure}
    \hfill
    \begin{subfigure}[b]{0.19\linewidth}
        \includegraphics[width=\linewidth]{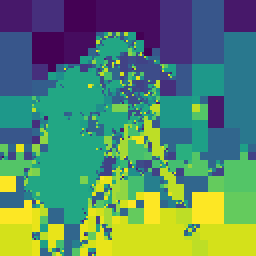}
        \caption{$\sigma = 15$}
    \end{subfigure}
    \begin{subfigure}[b]{0.19\linewidth}
        \includegraphics[width=\linewidth]{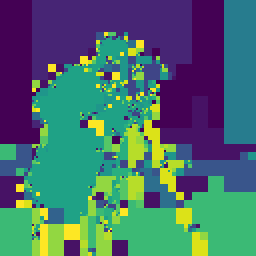}
        \caption{$\sigma = 25$}
    \end{subfigure}
    \caption{The MAP estimates of $\bm z$ and $T$. The color represents the value of $\bm z$. For different $\sigma$, the same color may represent different $\bm z$.}
    \label{fig:map_z_T}
\end{figure}

\noindent\textbf{Discussion.}
% \noindent\textbf{Performance degradation under high noise levels.}
Figure~\ref{fig:map_z_T} visualizes the MAP estimates of $\bm{z}$ and $T$ obtained as a byproduct of the denoising process. 
As $\sigma$ increases, the quadtree region segmentation becomes coarser. This might be a cause of the performance degradation under high noise levels. This problem might be addressed by modifying the hyperparameter settings of the prior distributions or adjusting the initialization of the VB procedure. 
% For instance, a prior that encourages region splitting based on changes in trend rather than changes in local noise level may be more appropriate. Alternatively, instead of using the observed image directly as the initialization, a pre-denoised image obtained by some auxiliary method could be employed. 
Since the primary contribution of this paper is the proposal of a novel denoising framework, tuning of these hyperparameters is left as a direction for future work.

% \noindent\textbf{Sensitivity to the choice of optimization method.}
While simple gradient ascent was adopted here, we found the results were sensitive to the choice of step size. This problem can be addressed by using other optimization methods, such as quasi-Newton methods (e.g., \cite{liu1989limited}) and Adam\cite{Adam}.
% Indeed, preliminary small-scale experiments confirmed that the truncated Newton conjugate-gradient (TNC) method yields improved performance, although it was not adopted in the main experiments due to its substantially higher computational cost. Further improvements may be achievable by employing adaptive optimization methods such as Adam. Since the objective function is expressed as a quadratic form, quasi-Newton methods are also a viable option. However, as is common in image-related optimization problems, the large number of variables to be optimized must be taken into account. 
Since the aim of this paper is to propose a flexible framework to which a variety of optimization methods can be applied, the selection of the most suitable optimizer within the framework is also left as a topic for future investigation.

\begin{table}[htb]
    \caption{Denoising performances on Set12.}
    \label{tab:conventional_comparison_set12}
    \centering
    \begin{tabular}{lcccc}
        \hline
        Method & $\sigma$ & RMSE & PSNR & SSIM \\
        \hline
        GF & $5$ & 4.996 & 34.16 & 0.8812 \\
        TV & $5$ & 3.983 & 36.17 & 0.9453 \\
        BM3D & $5$ & \textbf{3.214} & \textbf{38.02} & \textbf{0.9603} \\
        DIP & $5$ & 5.746 & 33.14 & 0.8915 \\
        N2V & $5$ & 8.088 & 30.23 & 0.8782 \\
        Proposed & $5$ & 3.827 & 36.48 & 0.9273 \\
        \hline
        GF & $10$ & 9.549 & 28.53 & 0.7118 \\
        TV & $10$ & 6.131 & 32.45 & 0.8996 \\
        BM3D & $10$ & \textbf{4.892} & \textbf{34.39} & \textbf{0.9268} \\
        DIP & $10$ & 6.223 & 32.38 & 0.8850 \\
        N2V & $10$ & 9.314 & 28.94 & 0.8383 \\
        Proposed & $10$ & 6.322 & 32.13 & 0.8604 \\
        \hline
        GF & $15$ & 10.263 & 27.91 & 0.6947 \\
        TV & $15$ & 7.626 & 30.54 & 0.8474 \\
        BM3D & $15$ & \textbf{6.159} & \textbf{32.41} & \textbf{0.8987} \\
        DIP & $15$ & 7.103 & 31.16 & 0.8594 \\
        N2V & $15$ & 10.099 & 28.18 & 0.8072 \\
        Proposed & $15$ & 10.893 & 27.39 & 0.6779 \\
        \hline
    \end{tabular}
\end{table}

%%%%%
\section{Concluding Remark}
%%%%%
We proposed the probabilistic image denoising framework. The experiments clarified the directions for future improvement, and we believe that our framework is promising.

\clearpage

%%%%%
\bibliographystyle{IEEEtran}
\bibliography{refs}
%%%%%

\clearpage

%%%%%
\appendices
%%%%%

%%%%%
\section{Overview of problem setup and graphical model} \label{section_append_model}
%%%%%

Figure \ref{fig:image_recovery} illustrates the overview of our problem setup. As described in Section \ref{section_proposed_model}, the parameters $\bm z$, $T$, $\bm \theta$, $\bm \tau$, and $\bm \pi$ are generated according to the prior distributions. Then, an original image $\bm v$ is generated according to the probabilistic image generative model. The observed image $\bm v'$ is obtained through a degradation process $p(\bm v' | \bm v)$. Our goal is to restore the original image $\bm v$ from the observed image $\bm v'$.

\begin{figure}[h]
    \centering
    \includegraphics[width=\linewidth]{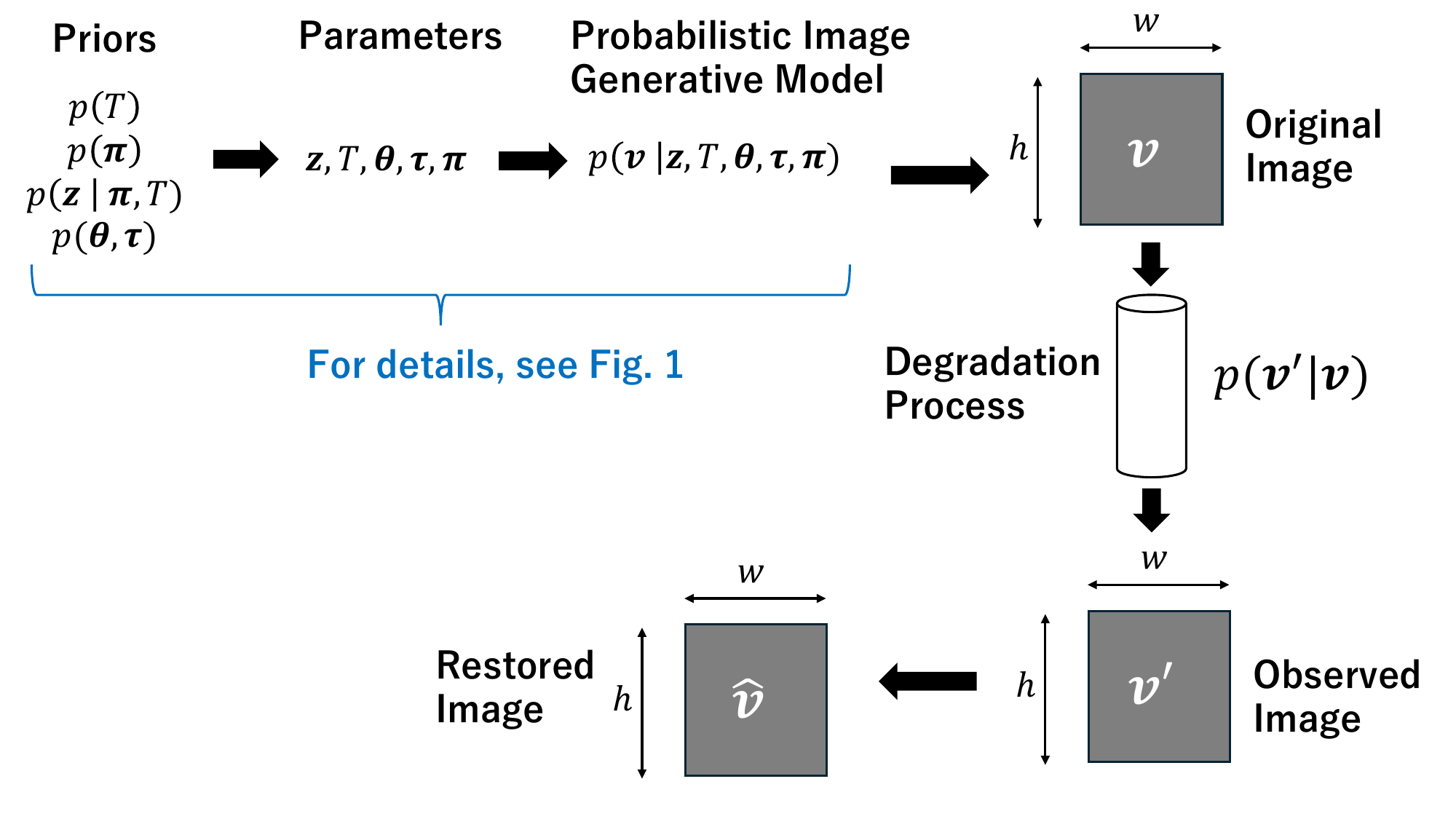}
    \caption{Overview of image denoising}
    \label{fig:image_recovery}
\end{figure}

Figure \ref{fig:graphical_model} shows the graphical model of our proposed model in Section \ref{section_proposed_model}. We denote observed variables by shading the corresponding nodes.

\begin{figure}[h]
    \centering
    \includegraphics[width=\linewidth]{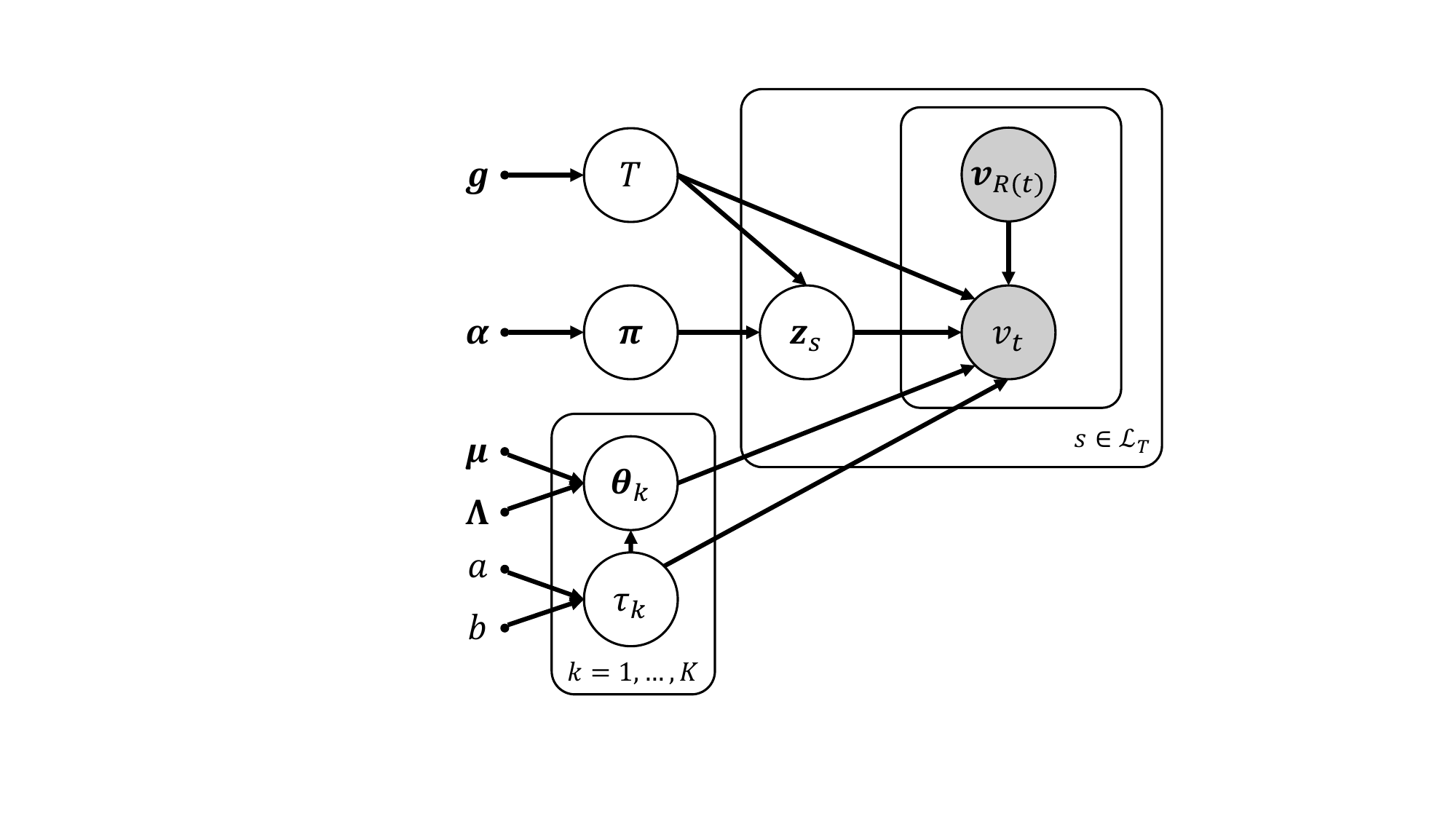}
    \caption{Graphical model}
    \label{fig:graphical_model}
\end{figure}

%%%%%
\section{Example of a reference pixel vector} \label{section_append_reference_vector}
%%%%%

As shown in Fig.\ \ref{fig:R}, we fix the position indices of the neighboring pixels of $v_t$ in advance. For example, when $D=4$, $R(t)=\{1,2,3\}$, i.e., the indices of the left, top-right, and top positions of the $t$-th pixel value $v_t$, and the 4-dimensional vector obtained by $R(t)$ and the constant term 1 becomes $\bm v_{R(t)}$.

\begin{figure}[t]
    \centering
    \includegraphics[width=0.8\linewidth]{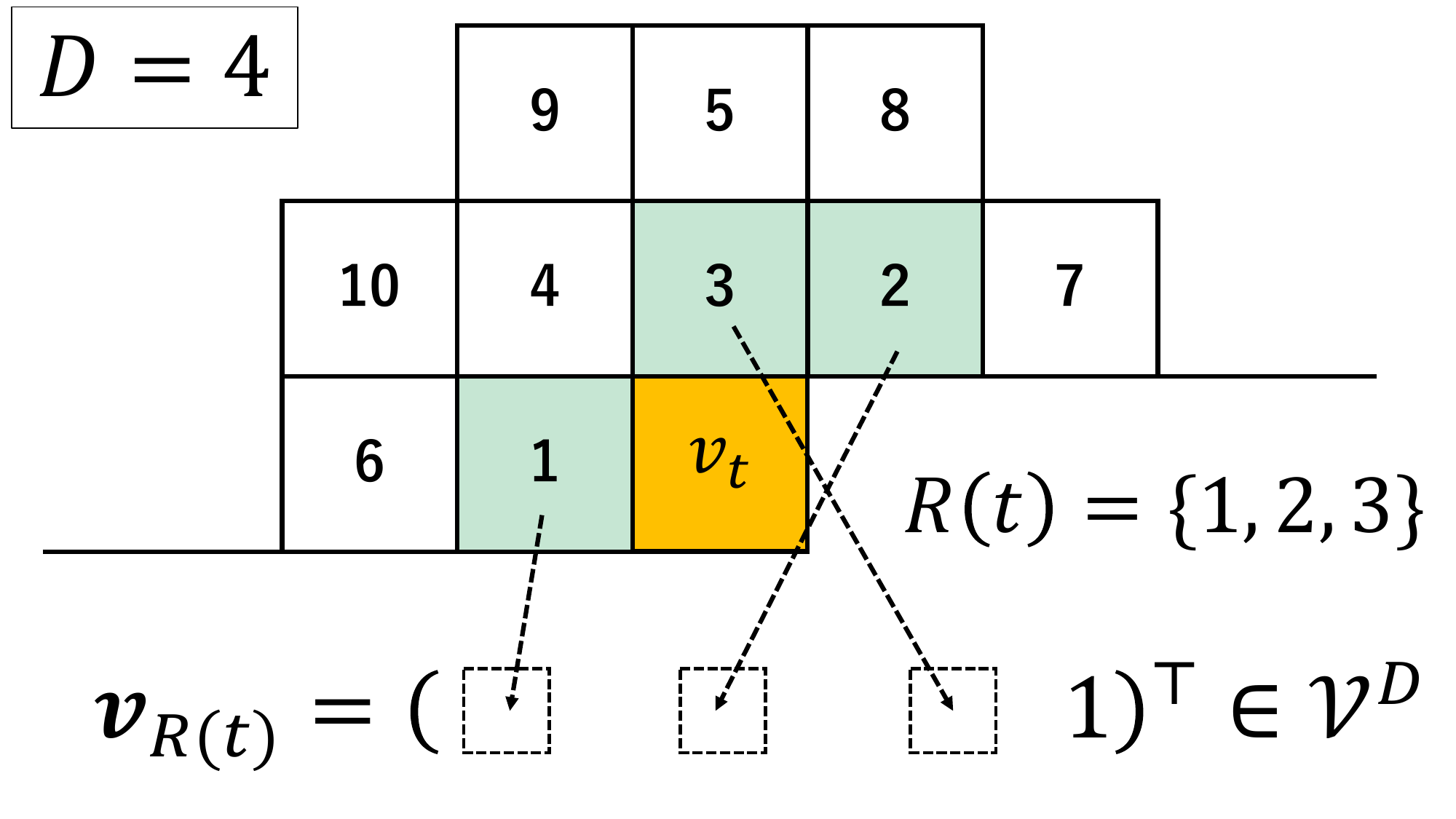}
    \caption{Example of $R(t)$ and $\bm v_{R(t)}$}
    \label{fig:R}
\end{figure}

%%%%%
\section{Proof of Proposition \ref{Prop_update_z_T}} \label{section_append_Prop_update_z_T}
%%%%%

From \eqref{q_star_z_T} and the assumptions in Section \ref{section_proposed_model}, 
\begin{align*}
    &\ln q(\bm z, T) \\
    &= \mathbb{E}_{q(\bm \theta, \bm \tau, \bm \pi)} \left[ \ln p(\bm v, \bm z, T, \bm \theta, \bm \tau, \bm \pi) \right] + \mathrm{const.} \\
    %&= \sum_{s \in \mathcal{I}_T} \ln g_s + \sum_{s \in \mathcal{L}_T} \ln (1-g_s) \\
    %& \qquad + \sum_{s \in \mathcal{L}_T} \sum_{k=1}^K z_{s,k} \mathbb{E}_{q^*(\bm \pi)} [\ln \pi_k] \nonumber \\
    %&\qquad + \sum_{s \in \mathcal{L}_T} \sum_{k=1}^K z_{s,k} \mathbb{E}_{q^*(\bm \theta, \bm \tau)} [\ln \mathcal{N}(\bm v_s | \bm V_s \bm \theta_k, (\tau_k \bm I_s)^{-1})] \\
    %& \qquad + \mathrm{const.} \\
    &= \sum_{s \in \mathcal{I}_T} \ln g_s + \sum_{s \in \mathcal{L}_T} \ln (1-g_s) \nonumber \\
    &\qquad + \sum_{s \in \mathcal{L}_T} \sum_{k=1}^K z_{s,k} \Big \{ \mathbb{E}_{q(\bm \pi)} [\ln \pi_k] \\
    & \qquad \qquad + \mathbb{E}_{q(\bm \theta, \bm \tau)} [\ln \mathcal{N}(\bm v_s | \bm V_s \bm \theta_k, (\tau_k \bm I_s)^{-1})] \Big \} + \mathrm{const.}
\end{align*}

We define $\rho_{s,k}$ and $\pi'_{s,k}$ as
\begin{align}
    \ln \rho_{s,k} &\coloneqq \mathbb{E}_{q(\bm \pi)}[ \ln \pi_k] \nonumber \\
    & \qquad + \mathbb{E}_{q(\bm \theta, \bm \tau)} [\ln \mathcal{N}(\bm v_s | \bm V_s \bm \theta_k, (\tau_k \bm I_s)^{-1})], \label{def_ln_rho} \\
    \pi'_{s,k} &\coloneqq \frac{\rho_{s,k}}{\sum_{k=1}^K \rho_{s,k}}. \nonumber
\end{align}
Then, 
\begin{align*}
    &\ln q(\bm z, T) = \\
    &\quad \sum_{s \in \mathcal{I}_T} \ln g_s + \sum_{s \in \mathcal{L}_T} \ln (1-g_s) + \sum_{s \in \mathcal{L}_T} \ln \sum_{k=1}^K \rho_{s,k} \\
    &\qquad + \sum_{s \in \mathcal{L}_T} \sum_{k=1}^K z_{s,k} \ln \pi'_{s,k} + \mathrm{const.},
\end{align*}
i.e., 
\begin{align*}
    & q(\bm z, T) \propto \\
    & \left( \prod_{s \in \mathcal{I}_T} g_s \right) \left( \prod_{s \in \mathcal{L}_T} (1-g_s) \sum_{k=1}^K \rho_{s,k} \right) \left( \prod_{s \in \mathcal{L}_T} \prod_{k=1}^K (\pi'_{s,k})^{z_{s,k}} \right).
\end{align*}

By using Lemma \ref{lemma_tree_distribution} in Appendix \ref{section_append_tree_distribution}, we get
\begin{align*}
    &q(\bm z, T) = \\
    &\quad \left( \prod_{s \in \mathcal{I}_T} g'_s \right) \left( \prod_{s \in \mathcal{L}_T} (1-g'_s) \right) \left( \prod_{s \in \mathcal{L}_T} \prod_{k=1}^K (\pi'_{s,k})^{z_{s,k}} \right),
\end{align*}
where $g'_s$ is defined as in \eqref{def_g_prime}. Therefore, we see that $q(\bm z, T)$ can be decomposed as $q(T) \prod_{s \in \mathcal{L}_T} q(\bm z_s)$. In Appendix \ref{section_append_calc_log_rho}, we will demonstrate that $\ln \rho_{s,k}$ can be calculated as in \eqref{eq:rho_closed_form}.

%%%%%
\section{Proof of Proposition \ref{Prop_update_pi_theta_tau}} \label{section_append_Prop_update_pi_theta_tau}
%%%%%
From \eqref{q_star_theta_tau_pi} and the assumptions in Section \ref{section_proposed_model}, 
\begin{align*}
    &\ln q (\bm \theta, \bm \tau , \bm \pi) \\
    &= \mathbb{E}_{q(\bm z, T)} \left[ \ln p(\bm v, \bm z, T, \bm \theta, \bm \tau, \bm \pi) \right] + \mathrm{const.} \\
    &= \sum_{k=1}^K \ln \mathcal{N} (\bm \theta_k | \bm \mu, (\tau_k \bm \Lambda)^{-1}) \mathrm{Gam}(\tau_k | a, b) + \ln \mathrm{Dir} (\bm \pi | \bm \alpha) \\
    &\quad+ \mathbb{E}_{q(\bm z, T)} \left[ \sum_{s \in \mathcal{L}_T} \sum_{k=1}^K z_{s,k} \ln \pi_k \right] \nonumber \\
    &\qquad + \mathbb{E}_{q(\bm z, T)} \left[ \sum_{s \in \mathcal{L}_T} \sum_{k=1}^K z_{s,k} \ln \mathcal{N}(\bm v_s | \bm V_s \bm \theta_k, (\tau_k \bm I_s)^{-1}) \right] \\
    &\qquad \qquad + \mathrm{const.}
\end{align*}

From this equation form, we see that $q(\bm \theta, \bm \tau, \bm \pi)$ can be decomposed as $q(\bm \theta, \bm \tau, \bm \pi) = q(\bm \theta, \bm \tau) q(\bm \pi)$.

First, $q(\bm \pi)$ is calculated as follows.
\begin{align*}
    &\ln q(\bm \pi) \\
    &= \sum_{k=1}^K (\alpha_k - 1) \ln \pi_k \\
    & \quad + \mathbb{E}_{q(T)} \left[ \sum_{s \in \mathcal{L}_T} \sum_{k=1}^K \pi'_{s,k} \ln \pi_k \right] + \mathrm{const.} \\
    &= \sum_{k=1}^K (\alpha_k - 1) \ln \pi_k \\
    & \quad + \mathbb{E}_{q(T)} \left[ \sum_{s \in \mathcal{S}_\mathrm{max}} I\{ s \in \mathcal{L}_T \} \sum_{k=1}^K \pi'_{s,k} \ln \pi_k \right] \\
    & \quad + \mathrm{const.} \\
    &= \sum_{k=1}^K (\alpha_k - 1) \ln \pi_k \\
    & \quad + \sum_{s \in \mathcal{S}_\mathrm{max}} q(s \in \mathcal{L}_T) \sum_{k=1}^K \pi'_{s,k} \ln \pi_k + \mathrm{const.} \\
    &= \sum_{k=1}^K \left(\alpha_k + \sum_{s \in \mathcal{S}_\mathrm{max}} q(s \in \mathcal{L}_T) \pi'_{s,k} - 1 \right) \ln \pi_k \\
    & \quad + \mathrm{const.}
\end{align*}
Thus, setting
\begin{align*}
    \alpha'_k \coloneqq \alpha_k + \sum_{s \in \mathcal{S}_\mathrm{max}} q(s \in \mathcal{L}_T) \pi'_{s,k},
\end{align*}
we conclude that
\begin{align*}
    q(\bm \pi) = \mathrm{Dir}(\bm \pi | \bm \alpha').
\end{align*}

Next, $q(\bm \theta, \bm \tau)$ is calculated as follows.
\begin{align*}
    &\ln q(\bm \theta, \bm \tau) \\
    &= \sum_{k=1}^K \ln \mathcal{N} (\bm \theta_k | \bm \mu, (\tau_k \bm \Lambda)^{-1}) \mathrm{Gam}(\tau_k | a, b) \nonumber \\
    &\qquad + \sum_{s \in \mathcal{S}_\mathrm{max}} q(s \in \mathcal{L}_T) \sum_{k=1}^K \pi'_{s,k} \ln \mathcal{N}(\bm v_s | \bm V_s \bm \theta_k, (\tau_k \bm I_s)^{-1}) \\
    & \qquad + \mathrm{const.} \\
    &= \sum_{k=1}^K \Biggl[ \frac{D}{2} \ln \tau_k - \frac{\tau_k}{2} (\bm \theta_k - \bm \mu)^\top \bm \Lambda (\bm \theta_k - \bm \mu) \\
    & \qquad \qquad + (a-1) \ln \tau_k - b \tau_k \nonumber \\
    &\qquad \qquad + \sum_{s \in \mathcal{S}_\mathrm{max}} q(s \in \mathcal{L}_T) \pi'_{s,k} \left( \frac{h_s w_s}{2} \ln \tau_k \right. \\
    & \qquad \qquad \qquad \left. -\frac{\tau_k}{2}(\bm v_s - \bm V_s \bm \theta_k)^\top (\bm v_s - \bm V_s \bm \theta_k) \right) \Biggr] + \mathrm{const.} \\
    &= \sum_{k=1}^K \Biggl[ \frac{D}{2} \ln \tau_k \\
    & \qquad - \frac{\tau_k}{2} \Biggl\{ \bm \theta_k^\top \left(\bm \Lambda + \sum_{s \in \mathcal{S}_\mathrm{max}} q(s \in \mathcal{L}_T) \pi'_{s,k} \bm V_s^\top \bm V_s \right) \bm \theta_k \\
    &\qquad - 2 \bm \theta_k^\top \left(\bm \Lambda \bm \mu + \sum_{s \in \mathcal{S}_\mathrm{max}} q(s \in \mathcal{L}_T) \pi'_{s,k} \bm V_s^\top \bm v_s \right) \\
    &\qquad + \bm \mu^\top \bm \Lambda \bm \mu + \sum_{s \in \mathcal{S}_\mathrm{max}} q(s \in \mathcal{L}_T) \pi'_{s,k} \bm v_s^\top \bm v_s \Biggr\} \\
    &\qquad + (a-1) \ln \tau_k - b \tau_k \\
    & \qquad + \left( \sum_{s \in \mathcal{S}_\mathrm{max}} q(s \in \mathcal{L}_T) \pi'_{s,k} \frac{h_s w_s}{2} \right) \ln \tau_k \Biggr] + \mathrm{const.}
\end{align*}
Threfore, setting $\bm \Lambda'_k$ and $\bm \mu'_k$ as in \eqref{def_lambda_prime} and \eqref{def_mu_prime}, we have
\begin{align*}
    &\ln q(\bm \theta, \bm \tau) \\
    &= \sum_{k=1}^K \Biggl[ \ln \mathcal{N}(\bm \theta_k | \bm \mu'_k, (\tau_k \bm \Lambda'_k)^{-1}) - \frac{\tau_k}{2} \\
    & \times \Biggl\{ \bm \mu^\top \bm \Lambda \bm \mu + \sum_{s \in \mathcal{S}_\mathrm{max}} q(s \in \mathcal{L}_T) \pi'_{s,k} \bm v_s^\top \bm v_s - (\bm \mu'_k)^\top \bm \Lambda'_k \bm \mu'_k \Biggr\} \nonumber \\
    &\qquad + (a-1) \ln \tau_k - b \tau_k \\
    &\qquad + \left( \sum_{s \in \mathcal{S}_\mathrm{max}} q(s \in \mathcal{L}_T) \pi'_{s,k} \frac{h_s w_s}{2} \right) \ln \tau_k \Biggr] + \mathrm{const.}
\end{align*}
Moreover, setting $a'_k$ and $b'_k$ as in \eqref{def_a_prime} and \eqref{def_b_prime}, we conclude that
\begin{align*}
    q(\bm \theta, \bm \tau) = \prod_{k=1}^K \mathcal{N}(\bm \theta_k | \bm \mu'_k, (\tau_k \bm \Lambda'_k)^{-1}) \mathrm{Gam}(\tau_k | a'_k, b'_k).
\end{align*}

%%%%%
\section{Calculation of $\ln \rho_{s,k}$} \label{section_append_calc_log_rho}
%%%%%
In this appendix, we show the calculation of \eqref{def_ln_rho}. First, from the property of Dirichlet distribution, 
\begin{align*}
    \mathbb{E}_{q(\bm \pi)}[ \ln \pi_k] = \psi (\alpha'_k) - \psi \left( \textstyle \sum_{k=1}^K \alpha'_k \right).
\end{align*}

Next, from the property of Gamma distribution and the lemma on quadratic form given in Appendix \ref{section_append_quadratic_form}, 
\begin{align*}
    &\mathbb{E}_{q(\bm \theta, \bm \tau)} [\ln \mathcal{N}(\bm v_s | \bm V_s \bm \theta_k, (\tau_k \bm I_s)^{-1})] \\
    &= \frac{h_s w_s}{2} \left( -\ln 2\pi + \mathbb{E}_{q(\tau_k)} [\ln \tau_k] \right) \\
    & \quad - \mathbb{E}_{q(\tau_k)} \left[ \frac{\tau_k}{2} \mathbb{E}_{q(\bm \theta_k | \tau_k)} \left[ (\bm v_s - \bm V_s \bm \theta_k)^\top (\bm v_s - \bm V_s \bm \theta_k) \right] \right] \\
    &= \frac{h_s w_s}{2} \left( -\ln 2\pi + \psi (a'_k) - \ln b'_k \right) \\
    & \quad - \mathbb{E}_{q(\tau_k)} \left[ \frac{\tau_k}{2} \Bigl( (\bm v_s - \bm V_s \bm \mu'_k)^\top (\bm v_s - \bm V_s \bm \mu'_k) \right. \\
    & \left. \qquad \qquad \qquad \qquad + \frac{1}{\tau_k} \mathrm{Tr} \{ \bm V_s^\top \bm V_s (\bm \Lambda'_k)^{-1} \} \Bigr) \right] \\
    &= \frac{h_s w_s}{2} \left( -\ln 2\pi + \psi (a'_k) - \ln b'_k \right) \\
    & \quad - \frac{a'_k}{2 b'_k} (\bm v_s - \bm V_s \bm \mu'_k)^\top (\bm v_s - \bm V_s \bm \mu'_k) \\
    & \qquad - \frac{1}{2} \mathrm{Tr} \{ \bm V_s^\top \bm V_s (\bm \Lambda'_k)^{-1} \}.
\end{align*}

Hence, $\ln \rho_{s,k}$ (which is defined in \eqref{def_ln_rho}) is expressed as \eqref{eq:rho_closed_form}.

%%%%%
\section{Lemma on tree distribution} \label{section_append_tree_distribution}
%%%%%

\begin{Lemma} \label{lemma_tree_distribution}
    For any $s \in \mathcal{S}_\mathrm{max}$, let $a_s > 0$ and $b_s > 0$. Suppose that a probability distribution of $T$ is represented as
    \begin{align}
        p(T) \propto \left( \prod_{s \in \mathcal{I}_T} a_s \right) \left( \prod_{s \in \mathcal{L}_T} b_s \right). \label{eq_distribution_T_without_normalization}
    \end{align}
    Then, the probability distribution of $p(T)$ (including the normalization constant) is give as
    \begin{align}
        p(T) = \left( \prod_{s \in \mathcal{I}_T} g_s \right) \left( \prod_{s \in \mathcal{L}_T} (1-g_s) \right), \label{eq:lemma1}
    \end{align}
    where
    \begin{align}
        \phi_s &\coloneqq \begin{cases}
            b_s + a_s \prod_{s_\mathrm{ch} \in \mathrm{Ch}(s)} \phi_{s_\mathrm{ch}}, & s \in \mathcal{I}_\mathrm{max}, \\
            b_s, & s \in \mathcal{L}_\mathrm{max},
        \end{cases} \label{eq:lemma1_phi_s} \\
        g_s &\coloneqq \begin{cases}
            \frac{a_s \prod_{s_\mathrm{ch} \in \mathrm{Ch}(s)} \phi_{s_\mathrm{ch}}}{\phi_s}, & s \in \mathcal{I}_\mathrm{max}, \\
            0, & s \in \mathcal{L}_\mathrm{max}. \label{eq:lemma1_g_s}
        \end{cases}
    \end{align}
    Note that $0 \leq g_s \leq 1$ for any $s \in \mathcal{S}_\mathrm{max}$.
\end{Lemma}

\begin{IEEEproof}
    From \cite[Theorem 1]{full_rooted_trees}, the normalization constant of \eqref{eq_distribution_T_without_normalization} is written as
    \begin{align*}
        \sum_{T \in \mathcal{T}} \left( \prod_{s \in \mathcal{I}_T} a_s \right) \left( \prod_{s \in \mathcal{L}_T} b_s \right) = \phi_{s_\lambda},
    \end{align*}
    where $\phi_{s_\lambda}$ is defined as in \eqref{eq:lemma1_phi_s}. Hence, $p(T)$ is represented as
    \begin{align}
        p(T) = \frac{1}{\phi_{s_\lambda}} \left( \prod_{s \in \mathcal{I}_T} a_s \right) \left( \prod_{s \in \mathcal{L}_T} b_s \right). \label{eq:lemma1_q_t_with_phi}
    \end{align}

    Now, we prove that \eqref{eq:lemma1} equals \eqref{eq:lemma1_q_t_with_phi} by substituting \eqref{eq:lemma1_phi_s} and \eqref{eq:lemma1_g_s} into the right-hand side of \eqref{eq:lemma1}:
    \begin{align*}
    &\left( \prod_{s \in \mathcal{I}_T} g_s \right) \left( \prod_{s \in \mathcal{L}_T} (1-g_s) \right) \\
    & \overset{(a)}{=} \left( \prod_{s \in \mathcal{I}_T} g_s \right) \left( \prod_{s \in \mathcal{L}_T \cap \mathcal{I}_\mathrm{max}} (1-g_s) \right) \left( \prod_{s \in \mathcal{L}_T \cap \mathcal{L}_\mathrm{max}} 1 \right) \\
    & \overset{(b)}{=} \left( \prod_{s \in \mathcal{I}_T} \frac{a_s \prod_{s_\mathrm{ch} \in \mathrm{Ch}(s)} \phi_{s_\mathrm{ch}}}{\phi_s} \right) \\
    & \quad \times \left( \prod_{s \in \mathcal{L}_T \cap \mathcal{I}_\mathrm{max}} \left( 1 - \frac{a_s \prod_{s_\mathrm{ch} \in \mathrm{Ch}(s)} \phi_{s_\mathrm{ch}}}{\phi_s} \right) \right) \\
    & \qquad \times \left( \prod_{s \in \mathcal{L}_T \cap \mathcal{L}_\mathrm{max}} \frac{\phi_s}{\phi_s} \right) \\  
    & \overset{(c)}{=} \left( \prod_{s \in \mathcal{I}_T} \frac{a_s \prod_{s_\mathrm{ch} \in \mathrm{Ch}(s)} \phi_{s_\mathrm{ch}}}{\phi_s} \right) \\
    & \quad \times \left( \prod_{s \in \mathcal{L}_T \cap \mathcal{I}_\mathrm{max}} \frac{\phi_s - a_s \prod_{s_\mathrm{ch} \in \mathrm{Ch}(s)} \phi_{s_\mathrm{ch}}}{\phi_s} \right) \\
    & \qquad \times \left( \prod_{s \in \mathcal{L}_T \cap \mathcal{L}_\mathrm{max}}  \frac{b_s}{\phi_s} \right) \\  
    &= \left( \prod_{s \in \mathcal{I}_T} \frac{a_s \prod_{s_\mathrm{ch} \in \mathrm{Ch}(s)} \phi_{s_\mathrm{ch}}}{\phi_s} \right) \left( \prod_{s \in \mathcal{L}_T \cap \mathcal{I}_\mathrm{max}} \frac{b_s}{\phi_s} \right) \\
    & \qquad \times \left( \prod_{s \in \mathcal{L}_T \cap \mathcal{L}_\mathrm{max}} \frac{b_s}{\phi_s} \right) \\  
    &= \left( \prod_{s \in \mathcal{I}_T} \frac{a_s \prod_{s_\mathrm{ch} \in \mathrm{Ch}(s)} \phi_{s_\mathrm{ch}}}{\phi_s} \right) \left( \prod_{s \in \mathcal{L}_T} \frac{b_s}{\phi_s} \right) \\
    &\overset{(d)}{=} \frac{1}{\phi_{s_\lambda}} \left( \prod_{s \in \mathcal{I}_T} a_s \right) \left( \prod_{s' \in \mathcal{L}_T} b_s \right),
\end{align*}
where $(a)$ follows from $g_s = 0$ for $s \in \mathcal{L}_\mathrm{max}$ (see Section \ref{section_tree_prior}); $(b)$ is due to \eqref{eq:lemma1_g_s}; $(c)$ follows from \eqref{eq:lemma1_phi_s}; and $(d)$ is due to the telescoping product, i.e., the terms $\phi_s$ except $\phi_{s_\lambda}$ are canceled out.
\end{IEEEproof}

%%%%%
\section{Lemma on quadratic form} \label{section_append_quadratic_form}
%%%%%

\begin{Lemma}
    Let $\bm x$ be a random variable with mean $\bm \mu$ and covariance matrix $\bm \Sigma$. Also, let $\bm A$, $\bm B$ be a matrix (where $\bm B$ is a symmetric matrix) and $\bm y$ be a vector such that the following sum and product can be calculated. Then, it holds that
    \begin{align*}
        & \mathbb{E}_{\bm x} [ (\bm A \bm x - \bm y)^\top \bm B (\bm A \bm x - \bm y) ] \\
        &\quad = (\bm A \bm \mu - \bm y)^\top \bm B (\bm A \bm \mu - \bm y) + \mathrm{Tr}\{ \bm A^\top \bm B \bm A \bm \Sigma \}.
    \end{align*}
\end{Lemma}

\begin{IEEEproof}
We have the following chain of equality:
    \begin{align*}
    &\mathbb{E}_{\bm x} [ (\bm A \bm x - \bm y)^\top \bm B (\bm A \bm x - \bm y) ] \\
    &= \mathbb{E}_{\bm x} [\bm x^\top \bm A^\top \bm B \bm A \bm x - 2 \bm x^\top \bm A ^\top \bm B \bm y] + \bm y^\top \bm B \bm y \\
    &= \mathbb{E}_{\bm x} [\mathrm{Tr}\{ \bm A^\top \bm B \bm A \bm x \bm x^\top \}] - 2 \bm \mu^\top \bm A ^\top \bm B \bm y + \bm y^\top \bm B \bm y \\
    &\overset{(a)}{=} \mathrm{Tr}\{ \bm A^\top \bm B \bm A (\bm \mu \bm \mu^\top + \bm \Sigma) \} - 2 \bm \mu^\top \bm A ^\top \bm B \bm y + \bm y^\top \bm B \bm y \\
    &= \bm \mu^\top \bm A^\top \bm B \bm A \bm \mu + \mathrm{Tr}\{ \bm A^\top \bm B \bm A \bm \Sigma \} - 2 \bm \mu^\top \bm A ^\top \bm B \bm y + \bm y^\top \bm B \bm y \\
    &= (\bm A \bm \mu - \bm y)^\top \bm B (\bm A \bm \mu - \bm y) + \mathrm{Tr}\{ \bm A^\top \bm B \bm A \bm \Sigma \},
    \end{align*}
    where $(a)$ follows from $\bm \Sigma = \mathbb{E}_{\bm x} [\bm x \bm x^\top] - \bm \mu \bm \mu^\top.$
\end{IEEEproof}

%%%%%
\section{Proof of Proposition \ref{Prop_update_v}} \label{section_append_Prop_update_v}
%%%%%

We explain the calculation of
\begin{align}
    \frac{\partial f(\bm v)}{\partial v_{i,j}} & = \frac{\partial \ln p(\bm v' | \bm v)}{\partial v_{i,j}} \nonumber \\
    & \quad + \frac{\partial  \mathrm{VL}(p(\bm v, \bm z, T, \bm \theta, \bm \tau, \bm \pi), q ( \bm z, T, \bm \theta, \bm \tau, \bm \pi))}{\partial v_{i,j}}. \label{eq_derivative}
\end{align}

Regarding the first term on the right-hand side of \eqref{eq_derivative}, we have
\begin{align}
    \ln p(\bm v' | \bm v) = \sum_{i=1}^h \sum_{j=1}^w \ln \mathcal{N}(v'_{i,j}|v_{i,j},\sigma^2)
\end{align}
because we assume \eqref{eq_gaussian_noise}.
Therefore,
\begin{align}
    \frac{\partial \ln p(\bm v' | \bm v)}{\partial v_{i,j}} = \frac{v'_{i,j} - v_{i,j}}{\sigma^2}. \label{eq_derivative_degrade_process}
\end{align}

Next, we consider the second term on the right-hand side of \eqref{eq_derivative}. Hereafter, we denote terms independent of $v_{i,j}$ as $\mathrm{const.}$ and denote $\mathrm{VL}(p(\bm v, \bm z, T, \bm \theta, \bm \tau, \bm \pi), q( \bm z, T, \bm \theta, \bm \tau, \bm \pi))$ as $\mathrm{VL}$.

From the same argument as in Appendix \ref{section_append_Prop_update_z_T}, \ref{section_append_Prop_update_pi_theta_tau}, and \ref{section_append_calc_log_rho}, we have
\begin{align}
    & \mathrm{VL} = \sum_{s \in \mathcal{S}_\mathrm{max}} q(s \in \mathcal{L}_T)  \nonumber\\
    & \times \sum_{k=1}^K \pi'_{s,k} \biggl( \frac{h_s w_s}{2} \left( -\ln 2\pi + \psi (a'_k) - \ln b'_k \right) \nonumber \\
    &- \frac{a'_k}{2 b'_k} (\bm v_s - \bm V_s \bm \mu'_k)^\top (\bm v_s - \bm V_s \bm \mu'_k) -\frac{1}{2} \mathrm{Tr} \{ \bm V_s^\top \bm V_s (\bm \Lambda'_k)^{-1} \} \biggr) \nonumber \\
    &+ \mathrm{const.}
\end{align}
Thus, it holds that
\begin{align}
    & \frac{\partial \mathrm{VL}}{\partial v_{i,j}} =  \sum_{s \in \mathrm{path}(v_{i,j})} q(s \in \mathcal{L}_T) \nonumber \\
    & \quad \times \sum_{k=1}^K \pi'_{s,k} \biggl(- \frac{a'_k}{2 b'_k} \frac{\partial}{\partial v_{i,j}}(\bm v_s - \bm V_s \bm \mu'_k)^\top (\bm v_s - \bm V_s \bm \mu'_k) \nonumber \\
    & \qquad \qquad \qquad \qquad \qquad -\frac{1}{2} \frac{\partial}{\partial v_{i,j}} \mathrm{Tr} \{ \bm V_s^\top \bm V_s (\bm \Lambda'_k)^{-1} \} \biggr). \label{eq_derivative_VL}
\end{align}
Hence, we consider the calculation of
\begin{align}
    &\frac{\partial}{\partial v_{i,j}}(\bm v_s - \bm V_s \bm \mu'_k)^\top (\bm v_s - \bm V_s \bm \mu'_k), \label{eq_derivative_inner_product} \\
    &\frac{\partial}{\partial v_{i,j}} \mathrm{Tr} \{ \bm V_s^\top \bm V_s (\bm \Lambda'_k)^{-1} \}  \label{eq_derivative_trace}
\end{align}
for each $s \in \mathrm{path}(v_{i,j})$.

First, we expand \eqref{eq_derivative_inner_product} as
\begin{align}
     \frac{\partial}{\partial v_{i,j}} \bm v_s^\top \bm v_s - 2 \frac{\partial}{\partial v_{i,j}} \bm v_s^\top \bm V_s \bm \mu'_k + \frac{\partial}{\partial v_{i,j}} (\bm \mu'_k)^\top \bm V_s^\top \bm V_s \bm \mu'_k. \label{eq_derivative_inner_product_expansion}
\end{align}
Then, the first term of \eqref{eq_derivative_inner_product_expansion} is
\begin{align}
    \frac{\partial}{\partial v_{i,j}} \bm v_s^\top \bm v_s 
    = \frac{\partial}{\partial v_{i,j}} (\cdots + v_{i,j}^2 + \cdots ) 
    =2 v_{i,j}. \label{first_term}
\end{align}
Next, the second term of \eqref{eq_derivative_inner_product_expansion} is
\begin{align}
    \frac{\partial}{\partial v_{i,j}} \bm v_s^\top \bm V_s \bm \mu'_k 
    = \left(\bm v_{R(i,j)} + \bm v_{s, \tilde{R}(i,j)} \right)^\top \bm \mu'_k \label{second_term}
\end{align}
because the matrix $\bm V_s$ can be written as
\begin{align}
\bm V_s
&= 
\begin{pmatrix}
& \vdots \\
&\bm v_{R(i,j)}^\top \\
&\bm v_{R(i,j+1)}^\top \\
& \vdots \\
&\bm v_{R(i+1,j-1)}^\top \\
&\bm v_{R(i+1,j)}^\top \\
& \vdots \\
\end{pmatrix} \label{matrix_V_s} \\
&=
\begin{pmatrix}
& \vdots \\
&\bm v_{R(i,j)}^\top \\
v_{i,j} & v_{i-1, j+2} & v_{i-1, j+1} & 1 \\
& \vdots \\
v_{i+1,j-2} & v_{i,j} & v_{i, j-1} & 1 \\
v_{i+1, j-1} & v_{i, j+1} & v_{i,j} & 1 \\
& \vdots \\
\end{pmatrix}. \nonumber
\end{align}
Finally, the third term of \eqref{eq_derivative_inner_product_expansion} is
\begin{align}
    \frac{\partial}{\partial v_{i,j}} (\bm \mu'_k)^\top \bm V_s^\top \bm V_s \bm \mu'_k 
    &= 2 (\bm \mu'_k)^\top \bm V_{s,\tilde{R}(i,j)} \bm \mu'_k \label{third_term}
\end{align}
because 
\begin{align*}
   &(\bm \mu'_k)^\top \bm V_s^\top \bm V_s \bm \mu'_k \\
   &\quad \overset{(a)}{=} \sum_{(k,l) \in \tilde{R}(i,j)} (\bm \mu'_k)^\top \bm v_{R(k,l)} \bm v_{R(k,l)}^\top \bm \mu'_k + \mathrm{const.}\\
   &\quad = \sum_{(k,l) \in \tilde{R}(i,j)} \left( \bm v_{R(k,l)}^\top \bm \mu'_k \right)^2+ \mathrm{const.},
\end{align*}
where $(a)$ follows from 
\begin{align}
    \bm V_s^\top \bm V_s = \sum_{(k,l) \in \tilde{R}(i,j)} \bm v_{R(k,l)} \bm v_{R(k,l)}^\top + \mathrm{const.} \label{eq_product_V_s}
\end{align}
Thus, from \eqref{eq_derivative_inner_product_expansion}, \eqref{first_term}, \eqref{second_term}, and \eqref{third_term}, we have
\begin{align}
    & \frac{\partial}{\partial v_{i,j}}(\bm v_s - \bm V_s \bm \mu'_k)^\top (\bm v_s - \bm V_s \bm \mu'_k) \nonumber \\
    & \quad = 2 \left \{ v_{i,j} - \left(\bm v_{R(i,j)} + \bm v_{s,\tilde{R}(i,j)} \right)^\top \bm \mu'_k \right. \nonumber \\
    & \left. \qquad \qquad \qquad \qquad + (\bm \mu'_k)^\top \bm V_{s,\tilde{R}(i,j)} \bm \mu'_k \right \}. \label{eq_derivative_inner_product_result}
\end{align}

Next, from \eqref{eq_product_V_s} and the definition of $\bm V_{s,\tilde{R}(i,j)}$, we have
\begin{align}
    &\frac{\partial}{\partial v_{i,j}} \mathrm{Tr} \{ \bm V_s^\top \bm V_s (\bm \Lambda'_k)^{-1} \} \nonumber \\
    & \quad = \mathrm{Tr} \left \{ \left( \bm V_{s,\tilde{R}(i,j)} + \bm V_{s,\tilde{R}(i,j)}^\top \right)
    (\bm \Lambda'_k)^{-1}
    \right \} \nonumber \\
    &\quad = 2 \mathrm{Tr} \left \{ \bm V_{s,\tilde{R}(i,j)} (\bm \Lambda'_k)^{-1} \right \}. \label{eq_derivative_Trace}
\end{align}

In conclusion, by substituting \eqref{eq_expression_q}, \eqref{eq_derivative_inner_product_result}, and \eqref{eq_derivative_Trace} into \eqref{eq_derivative_VL}, and then substituting this result and \eqref{eq_derivative_degrade_process} into the right-hand side of \eqref{eq_derivative}, we obtain Proposition \ref{Prop_update_v}.
\end{document}